\documentclass[journal]{IEEEtran}

\usepackage{graphicx,paralist,amsthm,amsmath,color,mathtools}
\usepackage{enumitem}
\newtheorem{proposition}{Proposition}
\newtheorem{definition}{Definition}
\usepackage{natbib,color}
\usepackage{tabularx}
\usepackage{tabu}
\usepackage{multirow}
\usepackage{colortbl}
\usepackage{amssymb,booktabs}
\usepackage[dvipsnames]{xcolor}
\usepackage{hyperref}
\DeclareMathOperator{\SO}{SO(2)}
\DeclareMathOperator{\logm}{logm}
\DeclareMathOperator{\expm}{expm}
\DeclareMathOperator{\relu}{ReLU}
\def\cmanifold{\mathbf{R}^+\!\times\!\SO}
\def\ang{\measuredangle}
\def\nzC{\mathbf{\widetilde{C}}}
\def\cz#1{\mathbf{#1}}
\def\rlarrows{\quad\mathrel{\substack{\longrightarrow \\[-.6ex] \longleftarrow}}\quad}
\def\conv{\ast}
\def\zconv{\,\widetilde{\ast}\,}

\definecolor{darkgreen}{rgb}{0, 0.9, 0}
\definecolor{brown}{rgb}{1, 0.5, 0}

\begin{document}
\title{SurReal: Complex-Valued Learning as Principled Transformations on a Scaling and Rotation Manifold}

\author{
\setlength{\tabcolsep}{10pt}
\begin{tabular}{ccc}
Rudrasis~Chakraborty &
Yifei~Xing & 
Stella~X.~Yu\\
\end{tabular}
\thanks{Manuscript received August 26, 2019; revised January 29, 2020 and June 10, 2020; accepted September 29, 2020. This work was supported in part by Berkeley Deep Drive and in part by DARPA. (Corresponding author: Stella X. Yu.)

The authors are with the International Computer Science Institute (ICSI), University of California, Berkeley (UC Berkeley), Berkeley, CA 94720, USA (e-mail: stellayu@berkeley.edu).

Color versions of one or more of the figures in this article are available online at \url{http://ieeexplore.ieee.org}.

Digital Object Identifier 10.1109/TNNLS.2020.3030565
}}

\maketitle

\begin{abstract}

Complex-valued data is ubiquitous in signal and image processing applications, and complex-valued representations in deep learning have appealing theoretical properties.  While these aspects have long been recognized, complex-valued deep learning continues to lag far behind its real-valued counterpart.

We propose a principled geometric approach to complex-valued deep learning.  
Complex-valued data could often be subject to arbitrary complex-valued scaling; as a result, real and imaginary components could co-vary.  Instead of treating complex values as two independent channels of real values, we  
recognize their underlying geometry:  We 
model the space of complex numbers as a product manifold of non-zero scaling and planar rotations.  Arbitrary complex-valued scaling naturally becomes a group of transitive actions on this manifold.  

We propose to extend the property instead of the form of real-valued functions to the complex domain.  We define convolution as weighted Fr\'{e}chet mean on the manifold that is {\it equivariant} to the group of scaling/rotation actions, and define distance transform on the manifold that is {\it invariant} to the action group.  The manifold perspective also allows us to define nonlinear activation functions such as tangent ReLU and $G$-transport, as well as residual connections on the manifold-valued data.

We dub our model {\it SurReal}, as our experiments on MSTAR and RadioML deliver high performance with only a fractional size of real-valued and complex-valued baseline models.

\end{abstract}

\begin{IEEEkeywords}
complex value, Riemannian manifold, Fr\'{e}chet mean, equivariance, invariance.
\end{IEEEkeywords}

\IEEEpeerreviewmaketitle

\def\figoverview#1{
\begin{figure*}[#1]
    \centering
    \includegraphics[clip,width=1\textwidth]{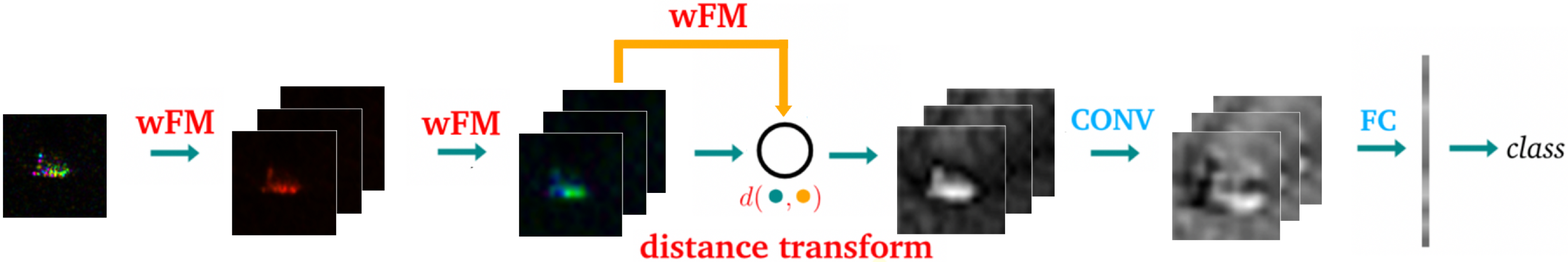}
    \caption{Sample architecture of our SurReal model. {\bf 1)} The input is a complex-valued image; each pixel is color coded with an HSV colormap according to its complex value, with the color intensity indicating the magnitude, the color hue indicating the phase, and the the full constant saturation.  {\bf 2)} The first two layers are our proposed complex-valued convolution in terms of weighted  Fr{\'e}chet mean filtering (wFM) on the  manifold derived from the polar form of complex numbers.  This convolution outputs complex-valued responses and it is equivariant to complex-valued scaling.  
    Each wFM layer could have multiple channels, each channel shown here as a complex-valued image.  Pixel-wise nonlinear activation functions such as our proposed tReLU and $G$-transport can be subsequently applied.
    {\bf 3)} The third layer is a distance transformation layer, where the manifold distance between the feature map and its wFM is computed.  This distance is invariant to complex-valued scaling.  
    {\bf 4)} Once the representation becomes real valued after the distance transformation layer, we could use any real-valued CNN layer functions for classification.  Shown here is the real-valued convolution (CONV) layer followed by the fully connected (FC) layer towards the final softmax classification.  With the built-in invariance to complex-valued scaling, our SurReal model can outperform real-valued baseline models on complex-valued data with a fraction of the baseline model size.
    }
    \label{fig:overview}
\end{figure*}
}

\section{Introduction}
\IEEEPARstart{W}{hile} deep learning has been widely successful in computer vision and machine learning \citep{lecun1998gradient, bengio2009learning, krizhevsky2012imagenet, he2016deep, lecun2015deep},
most techniques are only applicable to data that lie in a vector space. 
How to handle manifold-valued data and incorporate non-Euclidean geometry into  deep learning has become an active topic of research
\citep{cohen2016group, chakraborty2018h, chakraborty2018manifoldnet, esteves2017polar, bronstein2017geometric, chakraborty2018statistical}.

We are interested in extending deep learning to  complex-valued data, e.g., synthetic aperture radar (SAR) images in remote sensing, magnetic resonance (MR) images in medical imaging, or radio frequency (RF) signals in electrical engineering.
For such naturally complex-valued data, both the size (or magnitude) and the phase of a complex-valued measurement contain useful information.  For example, in SAR images, the magnitude encodes the amount of energy, whereas the phase variation indicates the object material and shape boundaries.  

Complex-valued data could also arise from the more informative complex-valued representation of naturally real-valued data.  The most notable examples are the Fourier spectrum and spectrum-based computer vision techniques ranging from steerable filters \citep{freeman1991design} to spectral graph embedding \citep{maire2016affinity,yu2011angular}.

The most common complex-valued deep learning approach is to simply apply real-valued deep learning methodology to the two-channel representation of complex-valued data $z=x+i\cdot y$ (where $i$ denotes the imaginary unit): the real component $x$ and the imaginary component $y$ are regarded as independent channels of the input. 

However, the independence assumption between the real and imaginary components does {\it not} hold in general.   For instance, in MR and SAR images, the pixel intensity value could be subject to arbitrary scaling by complex number $s=m\, e^{j\theta}$, where all the pixel values are simultaneously scaled in magnitude by $m$ and shifted in phase by $\theta$.  That is, any measurement $z$ is simply a representative of a whole class of possible equivalent measurements $\{z \,s: s=m\,e^{j\theta}: m>0, \forall \theta\}$.  Instead of being independent of each other, the real and imaginary components $(x,y)$ of $z$ co-vary in this equivalent class.

The co-variance of the two components of complex-valued data has not been exploited in complex-valued deep learning.  The common approach to learn a classifier invariant to scaling is to augment the training data with complex-valued scaling \citep{krizhevsky2012imagenet,dieleman2015rotation,wang2017joint}.  Such extrinsic data manipulation increases the amount of the training data and is rather ineffective: It takes a longer time to train the model, yet the invariance is not guaranteed.

Our goal is to develop the invariance to complex-valued scaling as an intrinsic property of the neural network itself.  We treat each complex-valued data sample as a point in a non-Euclidean space that respects the intrinsic geometry of complex numbers.  We propose new convolution and fully connected layer functions that can achieve equivariance and invariance to complex-valued scaling.
 
There has been a long line of works which define convolution in a non-Euclidean space by treating each data sample as a function in that space \citep{worrall2017harmonic,cohen2016group,cohen2017convolutional,esteves2017polar,chakraborty2018h,kondor2018generalization}. 
 
\figoverview{tp}
 
The challenge for defining such an equivariant convolution operator in the non-Euclidean space is the lack of a proper vector space structure.  In the Euclidean space, we can move from one point to another using an element from the group of translations; the standard convolution is thus equivariant to the action of the group of translations.  However, in the non-Euclidean space, e.g., a hypersphere, translation equivariance is no longer meaningful:  Translation is not the group to move from one point to another on a hypersphere, but rotation is.  

The concept of equivariance of an operator on a space is thus intimately related to the transitivity of a group of actions on that space.  We say that a group $G$ acts transitively on a space if there exists a $g\in G$ to go from one point to another on the space.   The group of translations acts transitively on  the Euclidean space, whereas the group of rotations acts transitively on a hypersphere.   The group that transitively acts on the  non-Euclidean space of the complex plane is non-zero scaling and planar rotations in the complex plane. 

The manifold view of convolution as an operator with equivariance to transitive actions on that space applies to both the domain and the range of data, e.g., for an image, its pixel coordinates define the domain, and its pixel intensities define the range.  Here we focus on the {\it range} space of data, in order to extend deep learning to complex-valued images and signals.

Our key insight is to represent a complex number in its polar form and define a Riemannian manifold on which complex scaling corresponds to the general transitive action group.  When a data sample lies on a Riemannian manifold, there are previously established results for deep learning:
\begin{itemize}[leftmargin=*]
\setlength{\itemsep}{0mm}
\item  Convolution defined by weighted Fr\'{e}chet mean (wFM) \citep{Frechet1948elements}  is equivariant to the group that naturally acts on that manifold \citep{chakraborty2018manifoldnet}.
\item  Since wFM is non-linear and acts like a contraction mapping \citep{Mallat2016understanding} analogous to ReLU or sigmoid, non-linear activation functions such as ReLU may not be needed.
\end{itemize}

We propose three types of complex-valued layer functions from the Riemannian geometric point of view:
\begin{enumerate}[leftmargin=*]
\setlength{\itemsep}{0mm}
\item{\bf wFM:}  a new convolution operator on the manifold for complex-valued data.  It is equivariant to complex-valued scaling.
The weights of wFM are to be learned.

\item{\bf Tangent ReLU}: a new nonlinear activation function that applies ReLU to the projections in the tangent space of the complex-valued manifold.  We also propose another option called {$G$-}transport, which transports a point on the complex manifold by an action in the scaling and rotation group.
\item{\bf Distance transform}: a new fully-connected layer operator that computes the manifold distance between a feature map and its wFM.  It is invariant to complex-valued scaling.  The weights of wFM are to be learned.
\end{enumerate}
The distance transform layer takes a complex-valued input to the real-valued domain, where any real-valued convolutional neural network (CNN) functions such as standard convolutions and fully connected (FC) layers can be subsequently used. 

Fig. \ref{fig:overview} shows a sample CNN architecture composed using our complex-valued layer functions.  A complex-valued image first passes through two wFM complex-valued convolutional layers, and then undergoes the distance transform.  The resulting real-valued distances are subsequently fed into a real-valued CNN classifier with one convolution layer and one FC layer.  Each convolutional layer is illustrated with a single channel response among a stack of many, with color images encoding complex-valued responses and grayscale images encoding real-valued responses.

Our complex-valued CNN has a group invariant property similar to the standard CNN on real-valued data.  Existing methods extend the real-valued counterpart to the complex domain based on the form of functions such as convolution or batch normalization \citep{bunte2012adaptive,trabelsi2017deep,pat:complex17}, 
not on the property of functions such as equivariance or linearity.  Our complex-valued CNN is composed of layer functions with  the desired equivariance and invariance properties that are essential for a real-valued CNN classifier in the Euclidean space; it is thus  a theoretically justified analog of the real-valued CNN. 

We compare our method with several baselines on two publicly available complex-valued datasets: MSTAR and RadioML.  Our model consistently outperforms the real-valued CNN baseline,  with fewer than 1\% on MSTAR and 3\% on RadioML of the baseline model parameters.  

We thus name our approach {\it SurReal} (pun intended): a surprisingly lean complex-valued model that beats the real-valued CNN model.
Our work has three major contributions.
\begin{enumerate}
\setlength{\itemsep}{0mm}
\item We propose novel complex-valued layer functions  with proven equivariance and invariance properties.
\item We extend our model to complex-valued residual CNNs. 
\item We validate our method on classification experiments.  Our SurReal CNNs outperform real- and complex-valued baselines at a fraction of their model sizes.
\end{enumerate}
These results demonstrate significant benefits of proposing CNN layer functions in terms of desirable intrinsic properties on the complex plane as opposed to applying the standard CNN to the 2D Euclidean embedding of complex numbers.

\section{Related Works}

Complex numbers are powerful representations and concepts in mathematics,  with intimate connections to geometry, topology, and differentiation \citep{complexBk1998}.   They have a wide range of applications in physics and engineering.  

\noindent
{\bf Complex-valued data representations} are widely used as a modeling choice to encode richer information than real-valued representations, especially for directional or cyclic data.
\citep{amin2009single} learns a mapping from a finite range of real values to the unit circle in the complex plane.
\citep{cadieu2012learning} trains a complex-valued sparse coding model to capture both edge structure and motion structure. 
\citep{yu:bright09,yu:ae12} combine the confidence and size of a measurement in a single complex value, and learns a global embedding from pairwise local measurements.
\citep{maire2016affinity} simultaneously encodes both grouping and figure-ground ordering relationships between neighboring pixels, and learns complex-valued pairwise pixel relationships from pixel-wise figure-ground annotations.
\citep{reichert2013neuronal} uses complex-valued neuronal units to model biologically plausible deep learning networks.
 \citep{scattering2013,argument2015} adopt wavelet transforms at earlier layers. \citep{arjovsky2016unitary} adopts unitary weight matrices in hidden layers for better learning performance.

\noindent
{\bf Traditional complex-valued data analysis} utilizes higher-order statistics such as variance fractal dimension trajectory \citep{Kinsner2010} and spectral analysis \citep{reichert1992automatic} to make adequate predictions. 

\noindent
{\bf Early neural network approaches} have already noted that complex values have many nice mathematical properties that real-value data do not have, e.g., the complex identity theorem.  Transformations from the input to the output can be more effectively learned with complex-valued networks instead of real-valued networks.  Various complex-valued activation functions have been explored, although with little demonstration of their success in real data settings \citep{complexBP1990,complexBP1992,complexBP1992sp,complexBP1997}.

\noindent
{\bf Recent neural network approaches} continue to build upon the theoretical advantages of complex-valued data to improve the convergence, stability, and generalization of neural networks \citep{nitta2002critical,hirose2012generalization},
and to facilitate the noise-robust memory retrieval mechanisms in capsule networks \citep{cvCapsNet2019}.  Real-valued layer functions have also been extended to the complex domain according to the form of the functions such as convolution, ReLU, and batch normalization \citep{bunte2012adaptive,trabelsi2017deep,pat:complex17,patThesis2018}.  Complex-valued deep learning has also been extended to quaternion neural networks, as quaternions generalize the concept of complex values from 2D to 3D \citep{quaternion2019}. 

\noindent
{\bf Recent graph convolution neural networks} open up new computational models in the complex domain
\citep{scarselli2008graph,bruna2013spectral}.  Since convolution in the spatial domain is equivalent to multiplication in the spectral domain, a natural extension of convolution to data defined on an arbitrary graph is to construct a convolutional filter in terms of multiplicative weights on the spectrum of the graph Laplacian \citep{bruna2013spectral}.   The spectrum of the graph Laplacian is real-valued if the graph is undirected, and complex-valued if it is directed \citep{singh2016graph}.

\noindent
{\bf Our SurReal complex-valued CNN} is unique in utilizing the geometric property of the complex numbers and approaching complex-valued learning as a special task of deep learning on Riemannian manifolds.  

Existing methods such as \citep{maire2016affinity, trabelsi2017deep, amin2009single,pat:complex17} treat a complex value as a vector in the Euclidean space of $\mathbf{R}^2$.  This choice, while straightforward, essentially destroys the covariant relationship between real and imaginary parts of a complex number.  Naturally complex-valued data such as SAR, MRI, and RF could be subject to complex-valued scaling without changing the underlying observation. 

To deal specifically with complex numbers, we first separate and acknowledge the extrinsic scaling effect by asking the convolution operator to be equivariant to complex-valued scaling.  For a CNN classifier, we design the distance transform layer to be invariant to complex-valued scaling.  

Our SurReal CNN classifier can focus entirely on the discriminative information between classes, without the need to build up additional scale invariance by repeatedly training on data augmented with complex-valued scaling.  
Our SurReal model is thus a surprisingly lean complex-valued model that beats the real-valued CNN model on complex-valued data.  An earlier preliminary version of this work was presented in \citep{surreal:cvprw19}.
\section{
A Scaling-Rotation Manifold
for the Geometry of Complex Numbers
}

A crucial property of complex-valued data is  complex-valued scaling ambiguity: The MRI or SAR images of the same scene could be related by the multiplication of a single complex number, depending on how the data is acquired. 
Complex-valued scaling can be captured by the scaling action on the magnitude and the rotation action on the phase.

Instead of treating the complex plane as the usual 2D Euclidean space, we identify the non-zero complex plane as the product manifold of positive magnitudes and planar rotations.  We show that scaling and rotation actions preserve the manifold distance defined on the non-zero complex manifold.

\noindent
{\bf Space of complex numbers.} 
Let $\mathbf{R}$ and $\mathbf{C}$ denote the field of real numbers and complex numbers respectively.  We have:
\begin{align}
{\bf z} = x+i\,y\in \mathbf{C}, \quad 
\forall x, y\in \mathbf{R}.
\end{align}
 According to this 2D real-valued representation $(x,y)$ of {\bf z},  $\mathbf{C}$ is a Riemannian manifold \citep{boothby1986introduction}.  The distance induced by the canonical Riemannian metric is:
\begin{align}
\label{eq:edist}
d({\bf z}_1,{\bf z}_2) = \sqrt{(x_2-x_1)^2+(y_2-y_1)^2},
\end{align}
the common Euclidean distance in the 2D complex plane.

\noindent
{\bf Polar form of complex numbers.}  
Any non-zero complex number can be uniquely represented in the polar form, in terms of its magnitude and phase.
\begin{definition}
\label{def1}
For $\forall{\bf z}\in\mathbf{C}$ and $\bf{z}\neq 0$, its polar form is:
\begin{align}
{\bf z} &= |{\bf z}|\exp\left(i\,\ang{{\bf z}}\right)\\
\bf{z} \in \mathbf{\widetilde{C}} & = \mathbf{C} \setminus \{0+i\,0\}\\
\text{magnitude: }|\bf{z}| & = \sqrt{x^2+y^2}\\
\text{ phase: }\ang{\bf{z}} & = \arctan(y,x)
\end{align}
where $\exp$ is the exponential function and $\arctan$ is the 2-argument 
arc-tangent function that gives the angle in the complex plane between the positive $x$ axis and the line from the origin to the point $(x,y)$.
\end{definition}

\noindent
{\bf Scaling-Rotation product manifold for $\nzC$.}  
Based on the polar form, we identify the non-zero complex plane $\widetilde{\mathbf{C}}$ as the product space of non-zero scaling and 2D rotations
\begin{align}
\widetilde{\mathbf{C}}
\Longleftrightarrow
\cmanifold
\end{align}
where $\mathbf{R}^+$ is the manifold of positive reals and $\SO$ is the manifold of planar rotations -- a rotation Lie group.

We define a bijective mapping $F$ that can go back and forth from  the complex plane $\widetilde{\mathbf{C}}$  to the manifold space $\cmanifold$:
\begin{align}
{\bf z} = |{\bf z}|\exp\left(i\,\ang{{\bf z}}\right)
&
\substack{F\\  \rlarrows \\ F^{-1}}
\left(|{\bf z}|, R(\ang{\bf z})\right)\\
R(\ang{\bf z})
&\quad =\quad
\begin{bmatrix*}[r] \cos(\theta) & -\sin(\theta) \\ \sin(\theta) & \cos(\theta)\end{bmatrix*}.
\end{align}
Both spaces are parameterized by magnitude and phase; the phase is turned into a complex number with $\exp(\cdot)$ for $\nzC$ and into a 2D rotation matrix with $R(\cdot)$ for $\cmanifold$.

\noindent
{\bf Manifold distance for $\nzC$.}
The exponential and logarithmic maps are respectively 
$\exp$ and $\log$ for $\mathbf{R}^+$, matrix exponential $\expm$ and matrix logarithm $\logm$ for $\SO$.
\begin{definition}
\label{logm}
The matrix exponential and logarithm of matrix $X$ are defined respectively as:
\begin{align*}
\expm(X)&=\sum_{n=0}^{\infty} \frac{X^{n}}{n!} \\
X &= \logm(Y) \text{ if and only if } Y = \expm(X).
\end{align*}
\end{definition}
\noindent
The distance on this product manifold in Eqn. \eqref{eq:edist} becomes:
\begin{align}
\label{eq:mdist}\displaystyle
\hspace{-5pt}
d\left(\mathbf{z}_1, \mathbf{z}_2\right) 
\!=\! \sqrt{\log^2\frac{|\mathbf{z}_2|}{|\mathbf{z}_1|}+\left\|
\logm\!\left(\!
R\left(\ang\mathbf{z}_2\right)
\phantom{\frac{.}{.}}\hspace{-8pt}
R(\ang\mathbf{z}_1)^{-1}
\!\right)\!
\right\|^2}.
\end{align} 

\noindent
{\bf Scaling-rotation is transitive on $\nzC$.}  The complex plane
$\widetilde{\mathbf{C}}$ as identified with $\cmanifold$ is a Riemannian homogeneous space \citep{helgason1962differential}.  We define  transitive actions that move a point around on the manifold \citep{dummit2004abstract}.
\begin{definition}
\label{def2}
Given a (Riemannian) manifold $\mathcal{M}$ and a group $G$ with identity element $e$, we say that $G$ acts on $\mathcal{M}$ (from the left) if there exists a mapping $L: G \times \mathcal{M} \rightarrow \mathcal{M}$ given by $\left(g,X\right) \mapsto g.X$ that satisfies two conditions:
\begin{enumerate} 
\item Identity: $L\left(e,X\right) = e.X = X$ 
\item Compatibility: $(gh).X = g.(h.X)$, $\forall g,h\in G$.
\end{enumerate}
An action is called transitive {\it if and only if} given $X, Y \in \mathcal{M}$, there exists an element $g \in G$, such that $Y = g.X$.  
\end{definition}
It is straightforward to verify that scaling and rotation in $\cmanifold$ satisfies the identity and compatibility conditions on $\mathbf{\widetilde{C}}$.  It is also a transitive group action: For any complex numbers ${\bf z}_1,{\bf z}_2\in \mathbf{\widetilde{C}}$, there always exists a relative scaling (of the magnitude) and rotation (of the phase) that maps  ${\bf z}_1$ to ${\bf z}_2$.  
\begin{proposition}
\label{theory:prop1}
The scaling-rotation Lie group $\cmanifold$ transitively acts on $\widetilde{\mathbf{C}}$ and the action $g$ to take  ${\bf z}_1$ to ${\bf z}_2$ is:
\begin{align}
g = \left(\frac{|{\bf z}_2|}{|{\bf z}_1|},\;
R(\ang\mathbf{z}_2)
R(\ang\mathbf{z}_1)^{-1}
\right) \in \cmanifold.
\end{align}
\end{proposition}

\noindent
{\bf Scaling-rotation is isometric on $\nzC$.}
We now show that scaling and rotation actions preserve our manifold distance.
\begin{proposition}
\label{theory:prop2}
The scaling and rotation Lie group is isometric on the complex plane $\nzC$:
$\forall \cz{z}_1,\cz{z}_2\in\nzC, g\in \cmanifold$.
\begin{align}
    d(g.\cz{z}_1,g.\cz{z}_2) = d(\cz{z}_1,\cz{z}_2)
\end{align}
where $d$ is the manifold distance defined in Eqn. (\ref{eq:mdist}).
\end{proposition}
\begin{proof}
We use the definition of $d$ and the property that the 2D rotation group $\SO$ is Abelian:  $\forall A,B\in\SO$, $AB=BA$.  Let $g=(r,A)\in \cmanifold$.  We have:
\begin{align*}
&d(g.\cz{z}_1,g.\cz{z}_2) \\
=&\sqrt{
\log^2\frac{r|\mathbf{z}_2|}{r|\mathbf{z}_1|}+\left\|
\logm\!\left(\!
AR\left(\ang\mathbf{z}_2\right)\,
\phantom{\frac{.}{.}}\hspace{-6pt}
\left(AR(\ang\mathbf{z}_1)\right)^{-1}
\!\right)\!
\right\|^2}\\
=&\sqrt{
\log^2\frac{|\mathbf{z}_2|}{|\mathbf{z}_1|}+\left\|
\logm\!\left(\!
R\left(\ang\mathbf{z}_2\right)A\,
\phantom{\frac{.}{.}}\hspace{-6pt}
\left(R(\ang\mathbf{z}_1)A\right)^{-1}
\!\right)\!
\right\|^2}\\
=&\sqrt{
\log^2\frac{|\mathbf{z}_2|}{|\mathbf{z}_1|}+\left\|
\logm\!\left(\!
R\left(\ang\mathbf{z}_2\right)\,
\phantom{\frac{.}{.}}\hspace{-6pt}
R(\ang\mathbf{z}_1)^{-1}
\!\right)\!
\right\|^2}= d(\cz{z}_1,\cz{z}_2),
\end{align*}
completing the proof.
\end{proof}

\def\figEqvar#1{
\begin{figure}[#1]\centering
\includegraphics[width=0.48\textwidth]{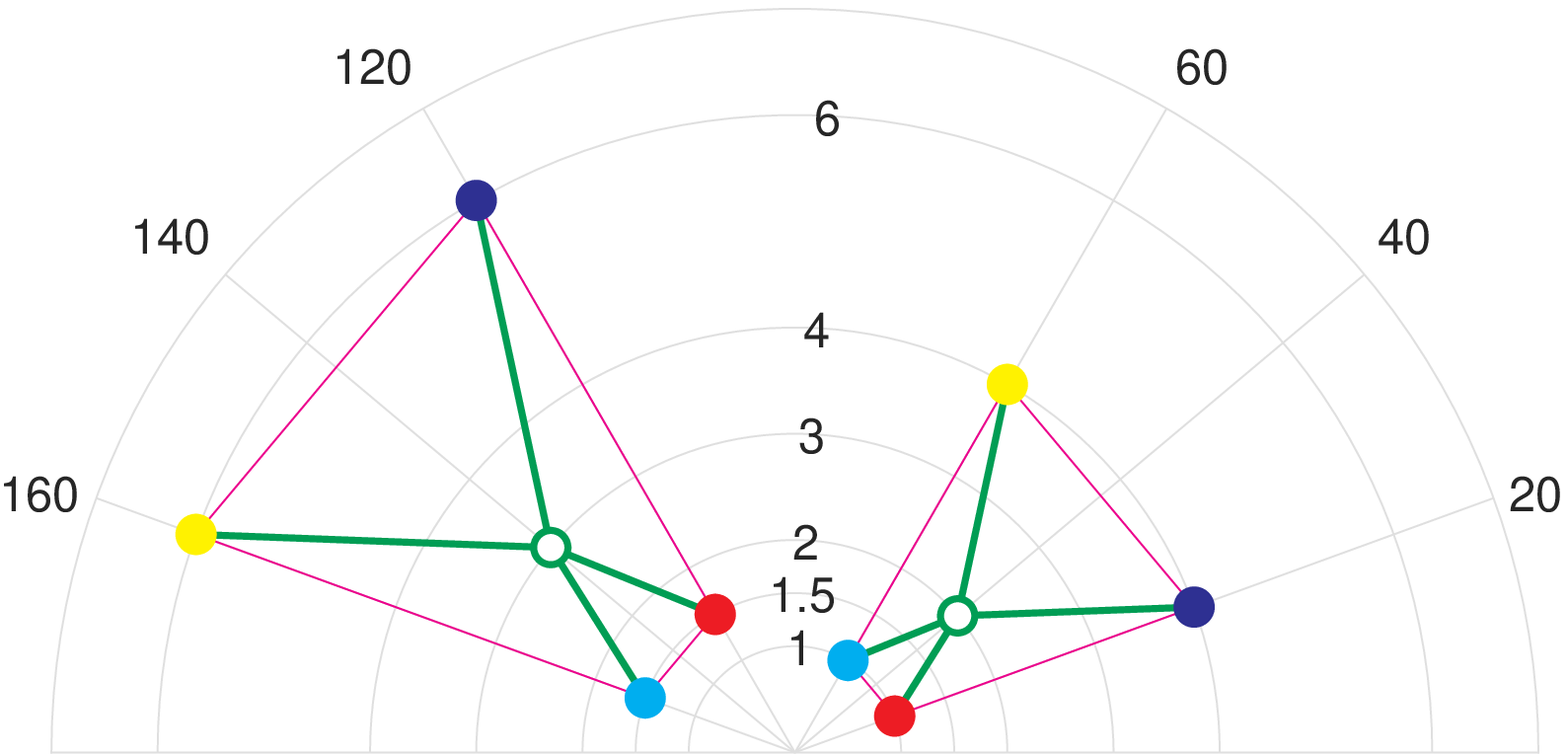}
\caption{Our complex-valued convolution in terms of wFM on the $\cmanifold$ manifold is equivariant to complex-valued scaling.  Consider 4 numbers in the complex plane, marked by 4 colored points on a small magenta trapezoid.  Their equally weighted wFM (marked by the green circle) sits inside the trapezoid at the geometric mean of their magnitudes and the mean of their phases.  When the 4 complex numbers are multiplied by $1.5\exp\left(i\,\frac{100\pi}{180}\right)$, the points are scaled by $1.5$ and rotated by $100^\circ$, moving to the larger trapezoid.  The new wFM is simply the old wFM transported by the same movement.
\label{fig:Eqvar}  
}
\end{figure}
}

\def\figtReLU#1{
\begin{figure}[#1]\centering
\setlength{\tabcolsep}{0pt}
\begin{tabular}{cc}
\includegraphics[clip,width=0.24\textwidth]{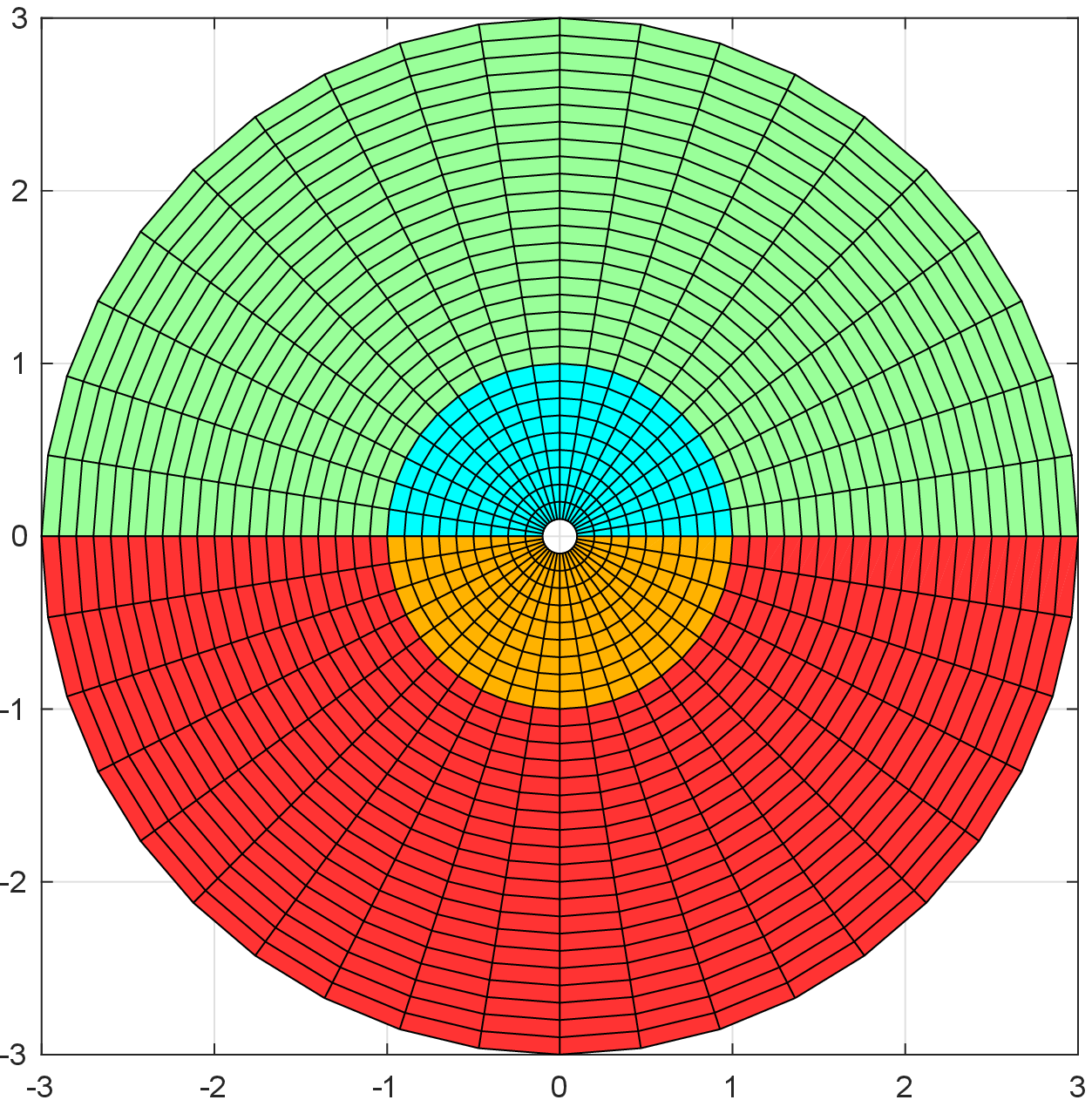}&
\includegraphics[clip,width=0.24\textwidth]{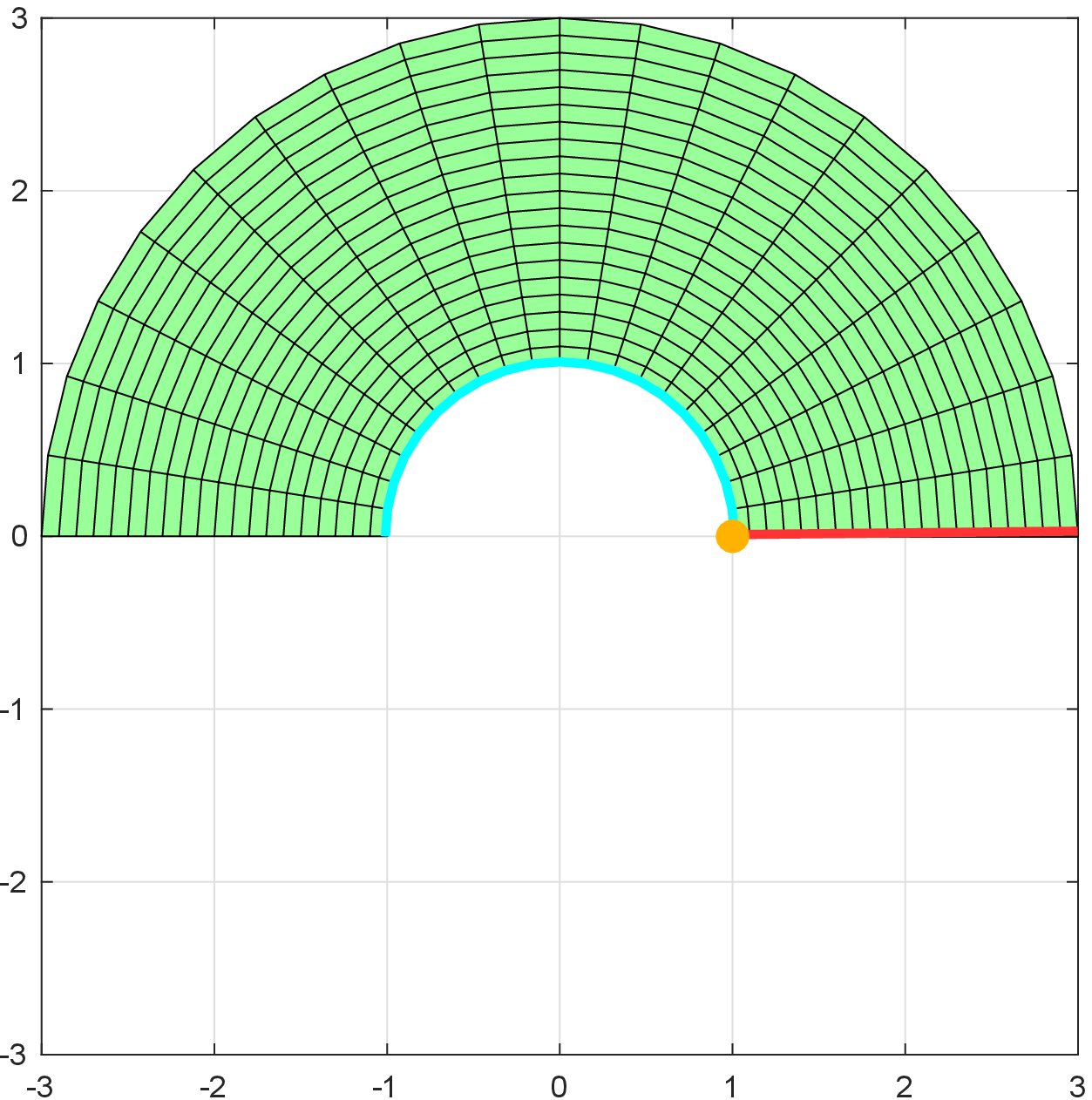}\\
\end{tabular}      
\caption{Our tangent ReLU extends ReLU from the real line to the complex plane $\nzC$, by applying ReLU in the tangent space of $\cmanifold$.   It rectifies the magnitude by $1$ and phase by $0$, creating four corresponding regions before (left) and after (right) the mapping.
{\bf\textcolor{darkgreen}{Green region}}: Points retain both their magnitudes and phases.  %
{\bf\textcolor{cyan}{Cyan region}}: Points retain their phases with their magnitudes rectified to 1.  %
{\bf\textcolor{red}{Red region}}: Points retain their magnitudes with their phases rectified to 0.  %
{\bf\textcolor{brown}{Brown region}}: Points are rectified in both the magnitude and the phase, all to the same point $1 + i \cdot 0$.  %
\label{fig:tReLU}
}
\end{figure}
}

\def\figDistInv#1{
\begin{figure}[#1]\centering
\includegraphics[width=0.48\textwidth]{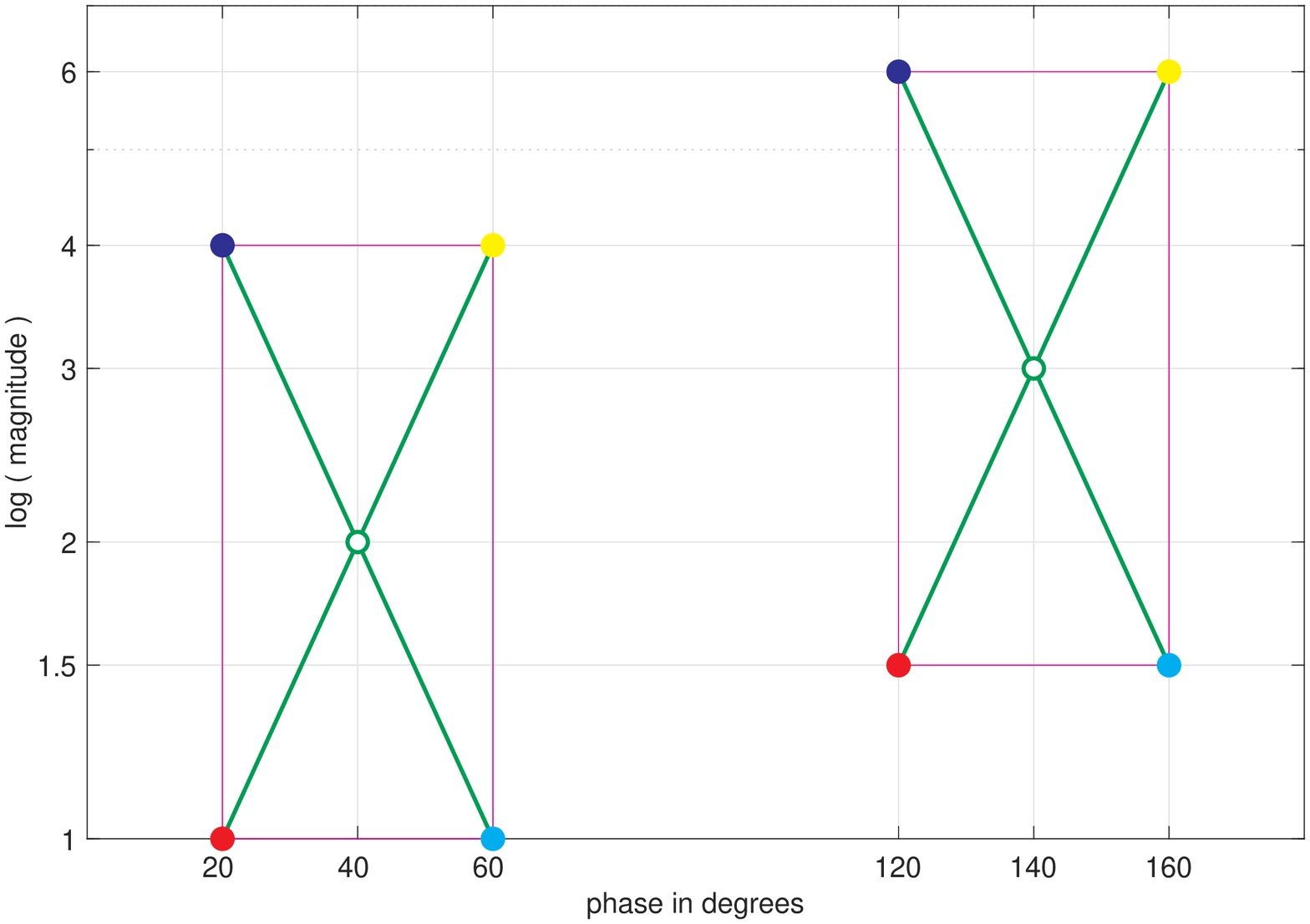}
\caption{Our distance transform is invariant to complex-valued scaling.  We plot Fig. \ref{fig:Eqvar} in the
tangent space of $\cmanifold$, with magnitude on a log scale for the $y$-axis and  phase for the $x$-axis.
The distance $d$ on the manifold is directly measured by the Euclidean distance in this space.   When the 4 complex numbers are multiplied by $1.5\exp\left(i\,\frac{100\pi}{180}\right)$, the points as well as their wFM are simply translated.  The distances between the points and their wFM thus remain the same.  Note: While this direct phase $\theta$ representation is intuitive, our rotation matrix $R(\theta)$ can more easily handle phase representation discontinuity at e.g. $\pm \pi$, where $\theta(-\pi,+\pi]$ for a unique determination of $\theta$.
\label{fig:DistInv}  
}
\end{figure}
}

\def\figResCNN#1{
\begin{figure}[#1]
\includegraphics[trim=0 0 10 0, width=0.48\textwidth]{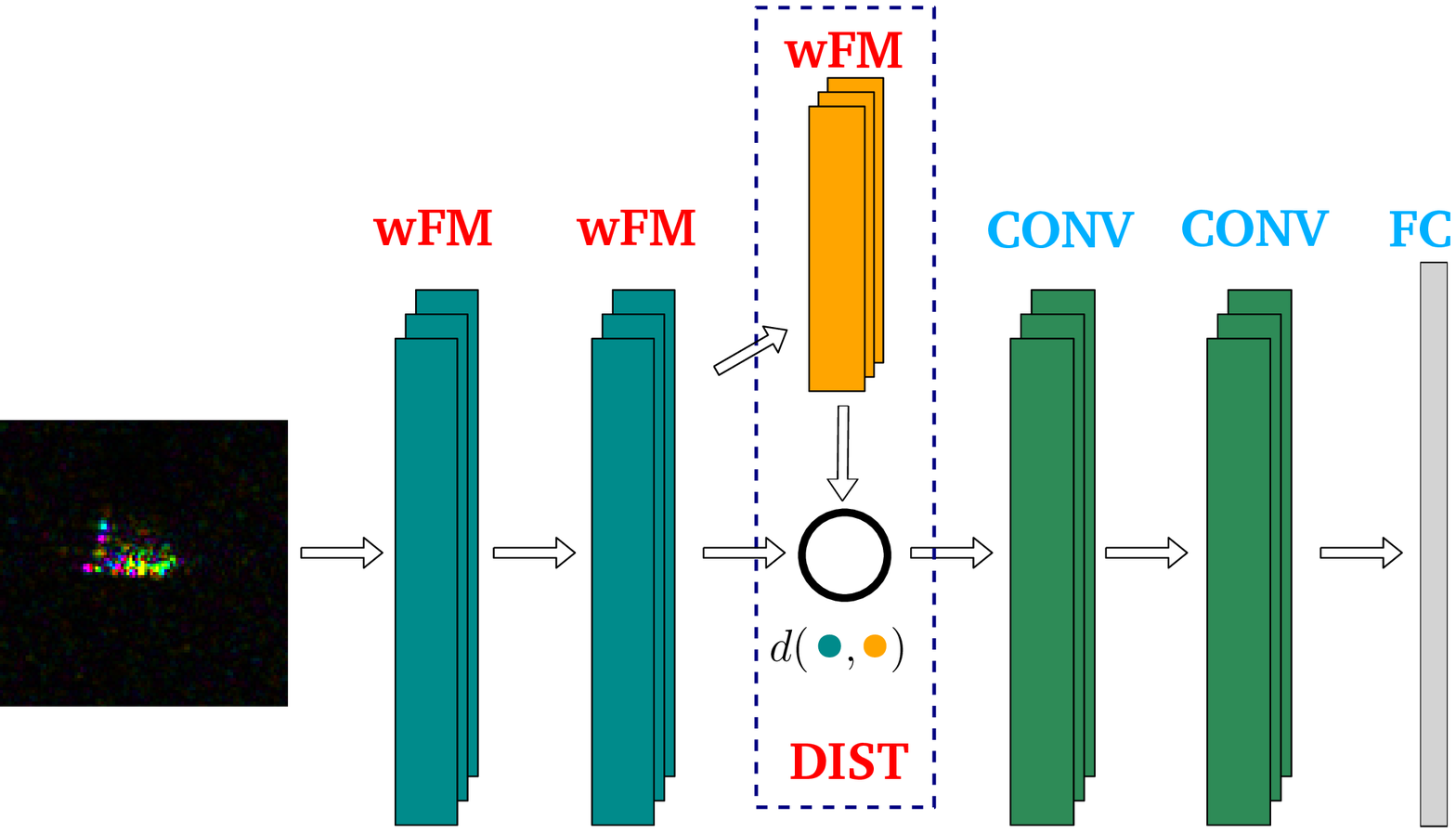}\\
\includegraphics[trim=0 0 5 0, width=0.48\textwidth]{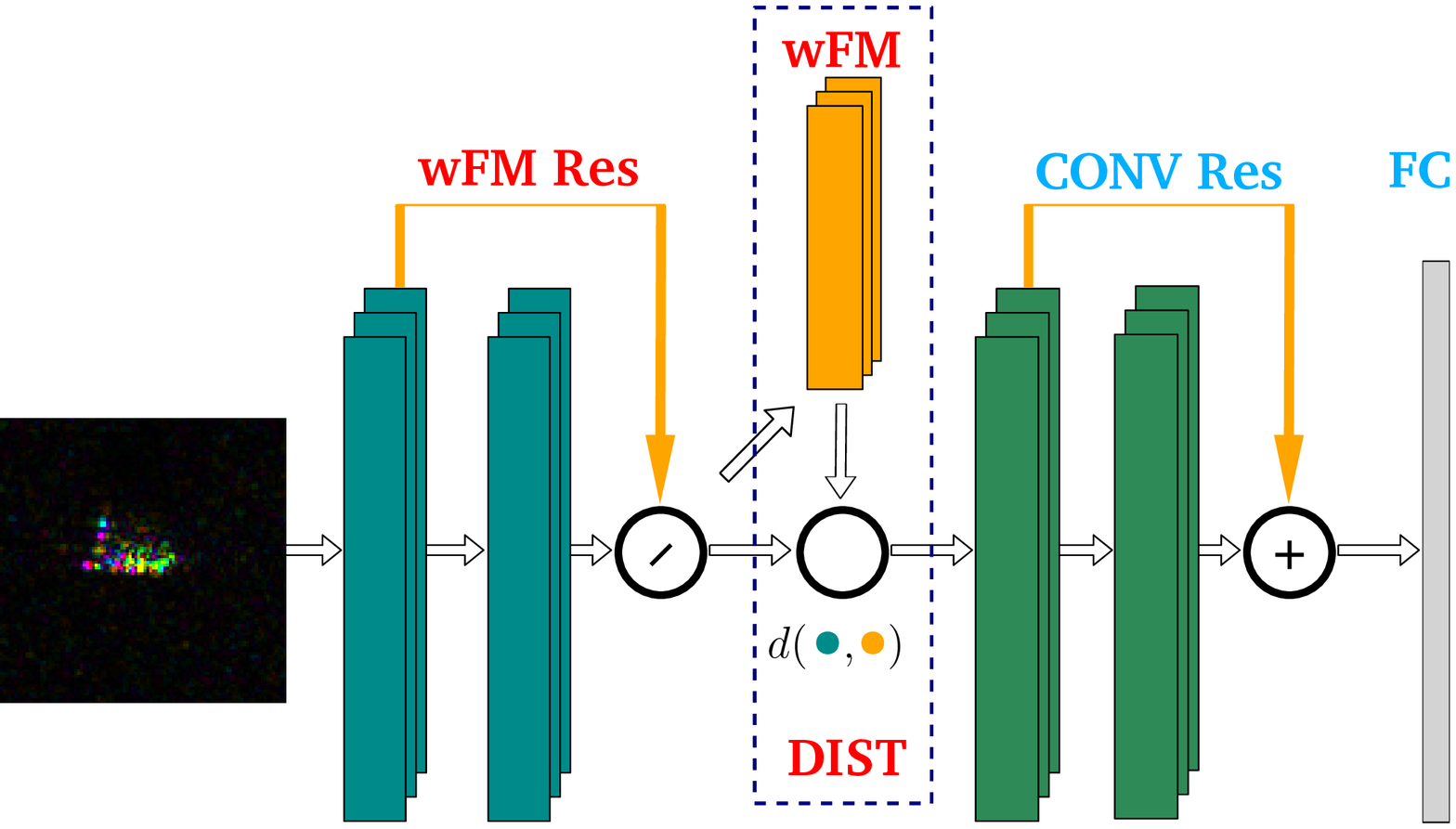}\\
\caption{Our complex-valued CNN with residual connections allows manifold valued representations to be combined across different depths. {\bf Top)} A sample CNN with complex-valued convolution (wFM), distance transform (DIST), real-valued convolution (CNN), and fully connected (FC) layers.
{\bf Bottom)} To create a residual structure (Res), we add skip connections between two adjacent convolutional layers.  We first spatially align their feature maps via convolution (wFM for complex-valued and CONV for real-valued) and then combine them via channel-wise  concatenation for complex-valued data and addition for the real-valued data.  The only difference between real-valued and complex-valued residual blocks is how the feature maps are combined: Addition is a vector space operator that does not apply to two points on a non-Euclidean manifold.
}
\label{fig:ResCNN}
\end{figure}
}

\section{CNN Layer Functions on Complex Manifold}
\label{theoryCNN}

We propose complex-valued CNN layer functions based on the scaling-rotation manifold view of complex numbers.  Each layer function transforms the representation with a certain property (e.g. equivariance or invariance) on the manifold.

We define first a convolution operator on the manifold that is {\it equivariant} to complex-valued scaling.  We then define nonlinear activation functions  and fully-connected layer functions that are {\it invariant} to complex-valued scaling.  A CNN composed with such layer functions would then be intrinsically invariant to complex-valued scaling.

\subsection{Complex-Valued Convolutional Layer Function}

The standard convolution, denoted by $\conv$ and defined by $n$ weights $\{w_k\}$ over $n$ neighbouring points $\{x_k\}$,  is simply the weighted average of real numbers in the Euclidean space:
\begin{align}
\{w_k\} \conv \{x_k\} = \sum_{k=1}^n w_k x_k.
\end{align}
We extend this concept to points on a manifold.

\noindent
{\bf Fr\'{e}chet mean on the manifold.} 
The weighted average of $n$ points on a Riemannian manifold is called the {\it weighted Fr{\'e}chet mean} (wFM) \citep{Frechet1948elements}.  We define the complex convolution, denoted by $\zconv$, as wFM on the scaling-rotation manifold for complex values (center circles in Fig. \ref{fig:Eqvar}):
\begin{align}
 \label{eq:zconv}
  \{w_k\} \zconv \{\cz{z}_k\} &= \text{wFM}(\{w_k\}, \{\cz{z}_k\})\\
\text{wFM}(\{w_k\}, \{\cz{z}_k\}) :&=\arg\min_{\cz{m}\in\nzC} \sum_{k=1}^n w_k d^2(\cz{z}_k,\cz{m})\\
\sum_{k=1}^{n} w_k &= 1, \quad w_k \ge 0 
\end{align}
where $d$ is the manifold distance in Eq. \eqref{eq:mdist}.  

We contrast our complex-valued wFM convolution  $\zconv$ with the standard real-valued convolution $\conv$.
\setlength{\itemsep}{4pt}
\begin{itemize}[leftmargin=*]
\item 
  While the output of $\zconv$ is complex-valued, the weights $\{w_k\}$ are real-valued, just like the weights for $\conv$.

\item While the weights of $\conv$ can be arbitrary, the weights of $\zconv$ are all nonnegative and summed up to 1.  This convexity constraint ensures that the wFM of $n$ points on a manifold stays on the manifold.
  
\item While the output of $\conv$ is simply the weighted average, the output of $\zconv$ is the minimizer to a weighted least squares problem, i.e., the data mean that minimizes the weighted variance.  There is no closed-form solution to wFM; however, there is a provably convergent $n$-step iterative solution \citep{iFME15}. 

\item If $d$ is the manifold distance in Eq. \eqref{eq:edist} for the Euclidean space which is also Riemannian, then  wFM has exactly the weighted average as its closed-form solution.  That is, our wFM convolution on the Euclidean manifold is reduced to the standard convolution, although with the additional convexity constraint on the weights.
\end{itemize}

\noindent
{\bf Convolutional wFM layer.}  As for the standard convolution, our weights $\{w_k\}$ for wFM are parameters learnable through stochastic gradient descent (SGD), with the additional convexity constraint on $\{w_k\}$. 

Each set of weights $\{w_k\}$ defines a single wFM channel in our convolutional layer, and each layer has multiple channels.  We follow the CNN convolution convention for images, where the convolutional kernel spans a local spatial neighbourhood but all the channels.  If there are 10 input channels of $50\times 50$ pixels, to produce 20 output channels with $5\times 5$ spatial kernels, we need to learn 20 sets of $5\times 5\times 10$ weights.

\figEqvar{tp}

\noindent
{\bf Equivariance of wFM to $\cmanifold$.}
We have shown that the scaling-rotation Lie group transitively acts on $\nzC$ and is isometric (Prop. \ref{theory:prop2}).  We use this result to prove that our wFM is equivariant to complex-valued scaling.
\begin{proposition}
\label{theory:prop3}
The complex-valued convolution $\zconv$ in Eq. \eqref{eq:zconv} is equivariant to the action of $\cmanifold$: $\forall g\in \cmanifold$,
\begin{align}
\{w_k\} \zconv \{g.\cz{z}_k\}  = g.\left (\{w_k\} \zconv \{\cz{z}_k\}\right).
\end{align}
\end{proposition}
\begin{proof}
  Let $g\in\cmanifold$ and $\cz{o}=\{w_k\} \zconv \{\cz{z}_k\}$.
$\forall \cz{m}\in \nzC$,
\begin{align*}
&\sum_{k=1}^n w_k d^2(g.\cz{z}_k,g.\cz{o}) = \sum_{k=1}^n w_k d^2(\cz{z}_k,\cz{o})\le \sum_{k=1}^n w_k d^2(\cz{z}_k,\cz{m})
\end{align*}
since $g$ perserves the distance and $\cz{o}$ is the minimizer over $\{\cz{z}_k\}$.  Therefore, $g.\cz{o}$ is the minimizer  over $\{g.\cz{z}_k\}$.
\end{proof}

Fig. \ref{fig:Eqvar} illustrates how wFM is equivariant to rotation and scaling.  For each trapezoid, the center circle marks the wFM of the four corner points.  If the trapezoid is transported using a particular scaling and rotation action, then the center wFM is also transported by the same action.

\subsection{Nonlinear Activation Functions}
The wFM convolution is a contractive mapping, an effect of a nonlinear activation function.  Nevertheless, for
stronger nonlinearity and acceleration in optimization during learning, we propose two activation functions from the manifold perspective: tangent ReLU and $G$-transport.

\noindent
{\bf Tangent ReLU (tReLU).}
The tangent space of a manifold is a vector space that contains the possible directions for tangentially passing through a point on the manifold.  It could be regarded as a local Euclidean approximation of the manifold.   A pair of  logarithmic and exponential maps establish the correspondence between the manifold and the tangent space.

We extend ReLU to the complex plane $\nzC$ by applying ReLU in the tangent space of $\cmanifold$ manifold.   Our tangent ReLU is composed of three steps.
\begin{enumerate}
\item Apply logarithmic maps to go from a point in $\nzC$ to a point in its tangent space.  The mapping is $\log$ for $r\in\mathbf{R}^+$,  and $\logm$ for $R(\theta) \in \SO$, which produces a skew symmetric matrix.  We choose the principal log map for $\SO$ in the range of $\theta \in (-\pi,\pi]$:
\begin{align}
\logm(R(\theta)) = \theta\cdot\begin{bmatrix*}[r]0&-1\\1&0\end{bmatrix*}.
\end{align}
\item Apply ReLU in the tangent space of $\nzC$.  ReLU is well defined for a real-valued scalar such as $\log(r)$ for $\mathbf{R}^+$:
\begin{align}
\relu(x) = \max(x,0),\quad \text{e.g.}, x = \log(r).
\end{align}    
We extend ReLU to $\logm(R(\theta))$ for $\SO$, since it is just the real-valued $\theta$ scaled by a constant matrix:
\begin{align}
\relu(\logm(R(\theta))) = \max(\theta,0)\cdot\begin{bmatrix*}[r]0&-1\\1&0\end{bmatrix*}.
\end{align}    
\item Apply exponential maps to come back to $\nzC$ from the tangent space. 
We can simplify the 3-step tReLU as:
\end{enumerate}
\begin{align}
r \!\overset{\text{tReLU}}{\mapsto}\!
&\exp\left(\relu(\log(r))\right)=\max(r,1)
\\
R(\theta) \!\overset{\text{tReLU}}{\mapsto}\!
&\expm\left(\relu(\logm(R(\theta)))\right)\!=\!R(\max(\theta,0))
\end{align}
Fig. \ref{fig:tReLU} shows that our manifold perspective of $\nzC$ leads to a non-trivial extension of ReLU from the real line to the complex plane, partitioning $\nzC$ into four regions, separated by $r=1$ and $\theta=0$: Those with magnitudes smaller than 1 are rectified to 1, and those with negative phases are rectified to 0.

\figtReLU{tp}

\noindent
{\bf $G$-transport.}
A nonlinear activation function is a general mapping that transforms the range of responses.  
We consider a general alternative which simply transports all the values in a feature channel via an action in the group $G=\cmanifold$.  We only need to learn one global scaling and rotation per feature channel, which corresponds to learning one complex-valued multiplier per channel at a certain depth layer.

\subsection{Fully-Connected Layer Functions} 
For classification tasks, having equivariance of convolution and range compression of nonlinear activation functions are not enough; we need the final representation of a CNN to be invariant to variations within each class.  

In a standard CNN classifier, the entire network is invariant to the action of the group of translations, achieved by the fully connected (FC) layer.  Likewise, we develop a FC layer  function on $\nzC$ that is invariant to the action of $\cmanifold$. 

Since the manifold distance $d$ is invariant to actions in $G=\cmanifold$, we propose the distance between each point of a set and their weighted Fr{\'e}chet mean, which is equivariant to $G$, as a new FC function on $\nzC$.  

\noindent
{\bf Distance transform FC Layer.}
Let the input be $s$ pixels of $c$ channels each.  We perform a global integration over all these $s\!\cdot\! c$ complex values $\{\cz{t}_k\}$.  Given $s\!\cdot\! c$ weights $\{w_k\}$ for these individual numbers, we first calculate their wFM $\cz{m}$ and then compute the distance $u_k$ from $\cz{t}_k$ to the mean $\cz{m}$:
\begin{align}
  \cz{m} &= \{w_k\} \zconv \{ \cz{t}_k\}
\label{eq:distm}\\
  u_{k} &= d(\cz{t}_k, \cz{m}).
\label{eq:distd}
\end{align}
The output is real-valued and of the same size as the input.   The weights $\{w_k\}$ are the parameters to be learned and there could be multiple sets of such weights at this layer.

\figDistInv{tp}

\begin{proposition}
\label{theory:prop4}
The distance to the wFM, defined in Eq. \eqref{eq:distm} and Eq. \eqref{eq:distd}, is invariant to the action of $G=\cmanifold$.
\end{proposition}
\begin{proof}
Per Propositions \ref{theory:prop2} and \ref{theory:prop3}, $\forall g\in G$, we have:
\begin{align*}
& d\left(g.\cz{t}_k, \{w_k\}\zconv \{g.\cz{t}_k\}\right)&\\
=\,&d\left(g.\cz{t}_k, g. \left(\{w_i\})\zconv \{\cz{t}_k\}\right)\right) 
&\text{equivariance of wFM}\\
=\,&d\left(\cz{t}_k, \{w_k\}\zconv\{\cz{t}_k\}\right)
&\text{invariance of distance}
\end{align*}
completing the proof.
\end{proof}
Fig. \ref{fig:DistInv} re-plots Fig. \ref{fig:Eqvar}  in the $(\log(r),\theta)$ space, which corresponds to the tangent space of $\cmanifold$ where the manifold distance can be directly visualized as the Euclidean distance.   When the four corners of the trapezoid are scaled and rotated, the trapezoid is simply translated along $(\log(r),\theta)$ axes.  The distance from points to their wFM   remain the same.

Since the output of the distance transform layer is real-valued, we can subsequently use any existing layer functions of a real-valued CNN classifier.  Fig.\ref{fig:overview} shows a sample architecture of our complex-valued CNN, where two successive wFM convolutional layers are followed by a distance transform FC layer, a standard convolutional layer, and an FC layer for final softmax classification.


\figResCNN{tp}

\subsection{Complex-Valued Residual Layer Function}

A standard CNN with residual layers such as ResNet \citep{he2016deep} outperforms the one without.  Residual layers are useful for preventing exploding/vanishing gradients in deep networks, by utilizing skip connections to jump over some layers.  The skip connections between layers add the outputs from previous layers to the outputs of stacked layers. 

While addition is natural for combining layers in the field of real numbers, it does not make sense in the field of complex numbers:  We can add two vectors in the Euclidean space, but we cannot add two points on a non-Euclidean manifold.

Here we propose a complex-valued residual layer function by retaining
the skip connection concept without the addition to combine the
outputs from different layers.  Let a feature layer $\cz{f}(s,c)$ be
specified by the number of pixels $s$ and the number of channels $c$.
Consider two feature layers $\cz{f}_1(s_1,c_1)$, $\cz{f}_2(s_2,c_2)$,
$s_1<s_2$, with one layer through skip connections.   In order to
combine them, we first use the wFM convolution to bring the spatial
dimension of $\cz{f}_2$ from $s_2$ to $s_1$ and then concatenate the
two sets of spatially aligned features:
\begin{align}
\text{align spatially: } &  \cz{f}_2(s_2,c_2) \overset{\zconv}{\to} \cz{\bar{f}}_2(s_1,c_2)\\
\text{concatenate: }& [\cz{f}_1(s_1,c_1) \, |\, \cz{\bar{f}}_2(s_1,c_2)]  \to \cz{f}(s_1,c_1+c_2)
\end{align}
Once combined, we can treat them as the input 
and apply any wFM convolution as desired.

Fig. \ref{fig:ResCNN} shows that 
we can simply replace two complex-valued convolution layers with such a residual block connecting and combining their outputs, and  build a residual complex-valued convolution network.  The only difference with the real-valued residual block is that the combination is channel-wise concatenation for a non-Euclidean manifold instead of addition for a vector space.

We can optionally further reduce the number of parameters for
convolution using the tensor ring decomposition \citep{zhao2016tensor}.  A $c$-dimensional
convolutional filter  $W$ of size $n_1\times\cdots\times n_c$ can be
decomposed into $c$ smaller rank $b$ tensors, each of the form
$T_k$ with size $b\times  n_k \times b$ such that  $\forall k_1, \cdots k_c$,
\begin{align}
W\!\!\left(k_1, \cdots, k_c\right) \!=\! \text{trace}\left(T_1(:,k_1,:)\!\times\! \cdots\! \times\! T_c(:,k_c,:)\right)
\end{align}
where $\times$ denotes matrix multiplication.  Such a tensor
factorization needs $b^2\sum_{k=1}^c n_k$ parameters for all the
tensors $\{T_k\}$ instead of $\prod_{k=1}^c n_k$ parameters for the
original $W$.   Tensor ring decomposition can achieve arbitrary approximation precision
(Theorem 2.2 on pp. 2299 in \citep{oseledets2011tensor}).

\def\tabMstarSize#1{
\begin{table}[#1]
    \centering
    \caption{
    MSTAR Dataset Size: Number of Images Per Class
    \label{tab:MstarSize}
    }
    \setlength{\tabcolsep}{4pt}
    \begin{tabular}
    {ccccccccccc}
    \toprule
C0&C1&C2&C3&C4&C5&C6&C7&C8&C9&C10\\
\midrule
1285& 429& 6694& 451& 1164& 1415& 573& 572& 573& 1401& 1159\\
\bottomrule
    \end{tabular}
\end{table}
}

\def\tabMstarAccuB#1{
\begin{table}[#1]
\centering
\caption{
Model Accuracy Comparison on MSTAR-L and MSTAR-S
}
\label{tab:MstarAccuB}
\begin{tabular}{l|c|c|c|c}
\toprule
\bf Test Accuracy (\%) & \textbf{ResNet50} & \textbf{DCN} &  \textbf{SurReal}&  \textbf{SurReal-Res}\\
\midrule
{MSTAR-L} & $99.1$ & $98.9$  & $99.1$ & $\bf 99.2$\\
\midrule
{MSTAR-S} & $97.4$ & $93.3$ & $97.7$ & $\bf 98.4$ \\
\midrule
{Difference S$-$L} & $-1.7$  & $-6.6$ & $-1.4$ & $\bf -0.8$ \\
\bottomrule
\end{tabular}\\[5pt]
\begin{tabular}{@{}p{0.48\textwidth}@{}}
Our SurReal CNN and its residual version consistently outperform the real-valued baseline ResNet50 and complex-valued baseline DCN.  SurReal-Res retains the  accuracy the most when the overall data size is reduced, demonstrating better generalization capability from small data.
\end{tabular}
\end{table}
}

\def\figMstarConf#1{
\begin{figure}[#1]\centering
\setlength{\tabcolsep}{0pt}
\begin{tabular}{cc}
SurReal & SurReal-Res\\
\includegraphics[clip,trim=60 10 82 22, width=0.24\textwidth]{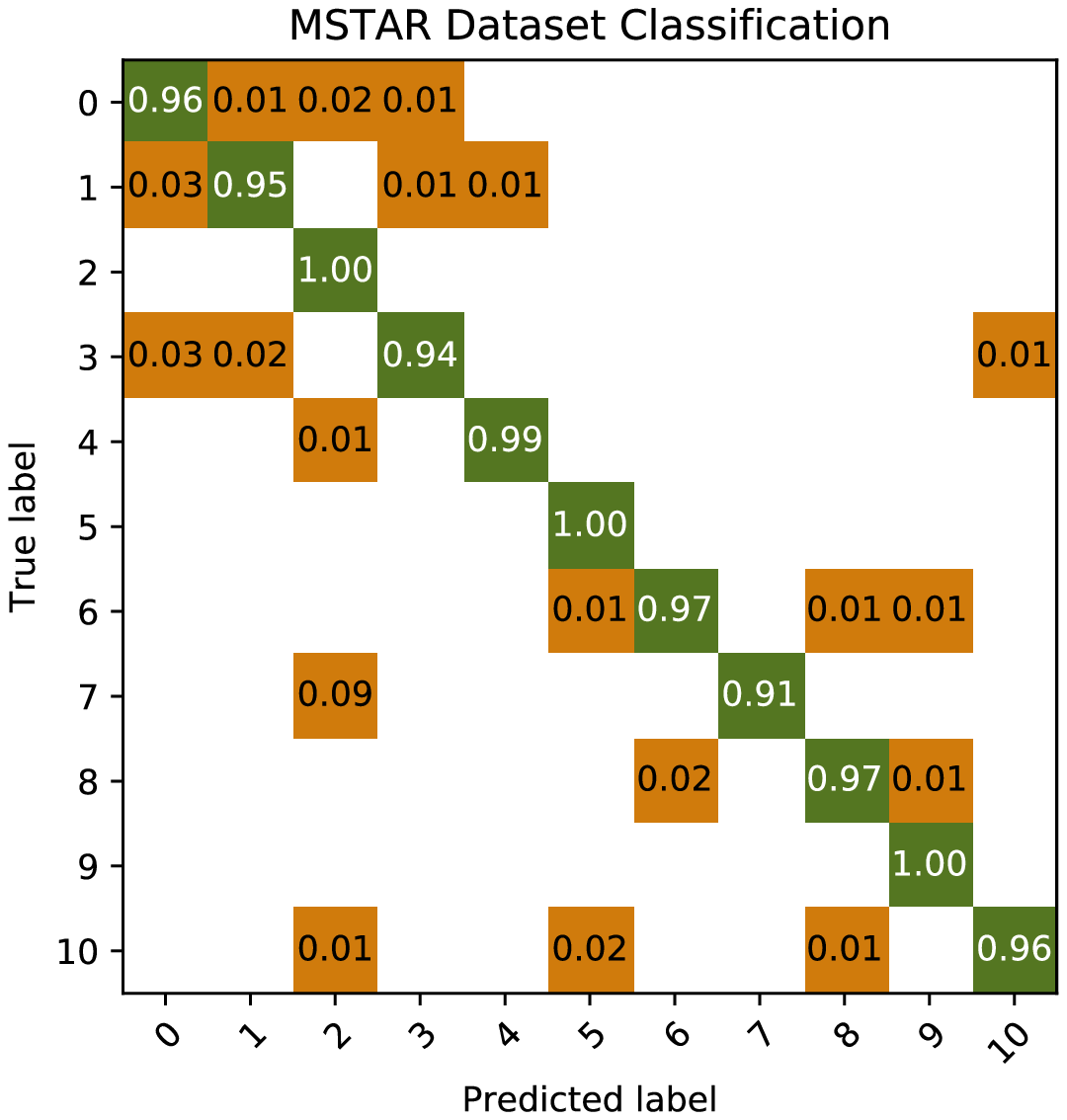}&
\includegraphics[clip,trim=60 10 82 22, width=0.24\textwidth]{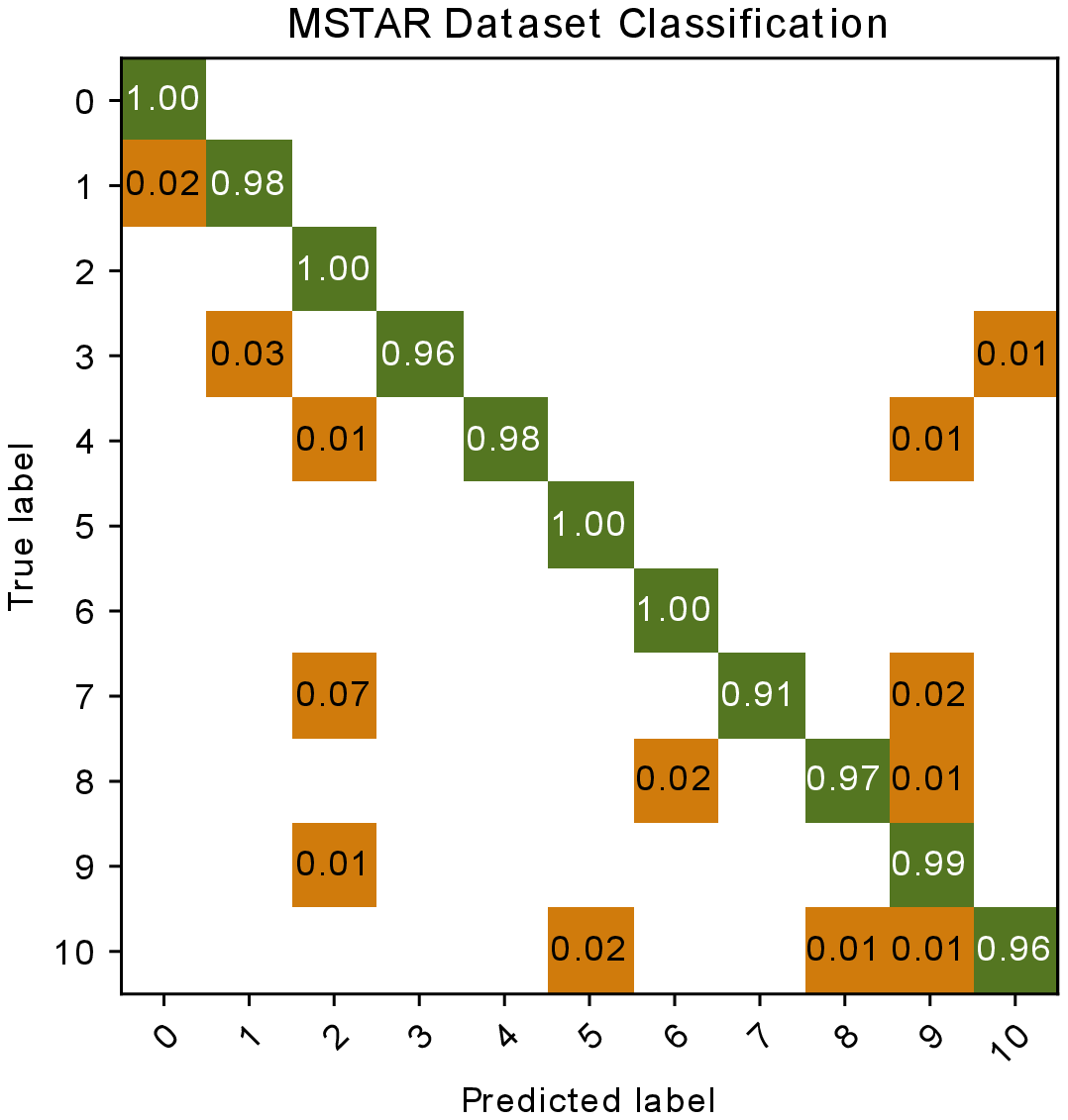}\\
\end{tabular}
\caption{The confusion matrix for our SurReal CNN (left) and SurReal-Res CNN (right) on MSTAR-S.  {\bf 1)} The dataset is small and highly imbalanced across classes (Fig. \ref{fig:mstarSets} Right).  The model achieves better accuracy for classes with more instances: The accuracy is 100\% for C2 and 94\% for C1 and C3.  Overall, the accuracy gap is small, considering the size of C1 and C3 is only 8\% of the size of C2.
{\bf 2)} The residual connections help further clear up the confusion between classes.  The matrix becomes more strongly diagonal.
}
\label{fig:MstarConf}
\end{figure}
}

\def\figMstarConvResponse#1{
\begin{figure*}[#1]\centering
\includegraphics[width=0.99\textwidth]{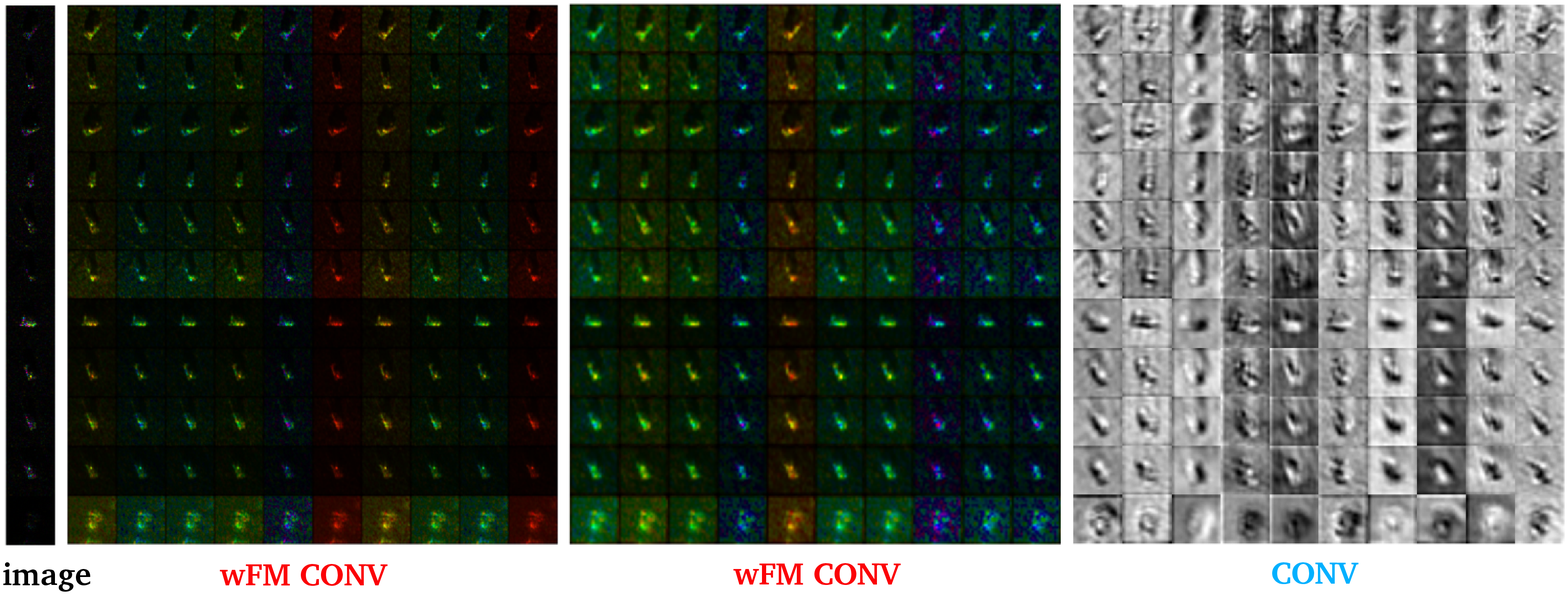}
\caption{Sample channel responses of our SurReal CNN on MSTAR-S images.  {\bf 1)} There are 11 rows, with Row $k$ containing an instance for class $k$.  Column 1 shows the input image.  Columns 2-4 contain 10 channel responses in the first three convolutional layers: the first two are complex-valued wFM convolution and the last is the real-valued convolution.  {\bf 2)}  All the complex values are displayed in fully saturated color using the cyclic HSV colormap, with the magnitude as the intensity and the phase as the hue, whereas the real values are displayed in gray using the grayscale colormap.  As the input image goes through the convolutional layers in the SurReal CNN, the distinction between instances in individual classes becomes clearer.}
\label{fig:MstarConvResponse}
\end{figure*}}

\def\figMstarModelSizes#1{
\begin{figure}[#1]\centering
\includegraphics[width=0.48\textwidth]{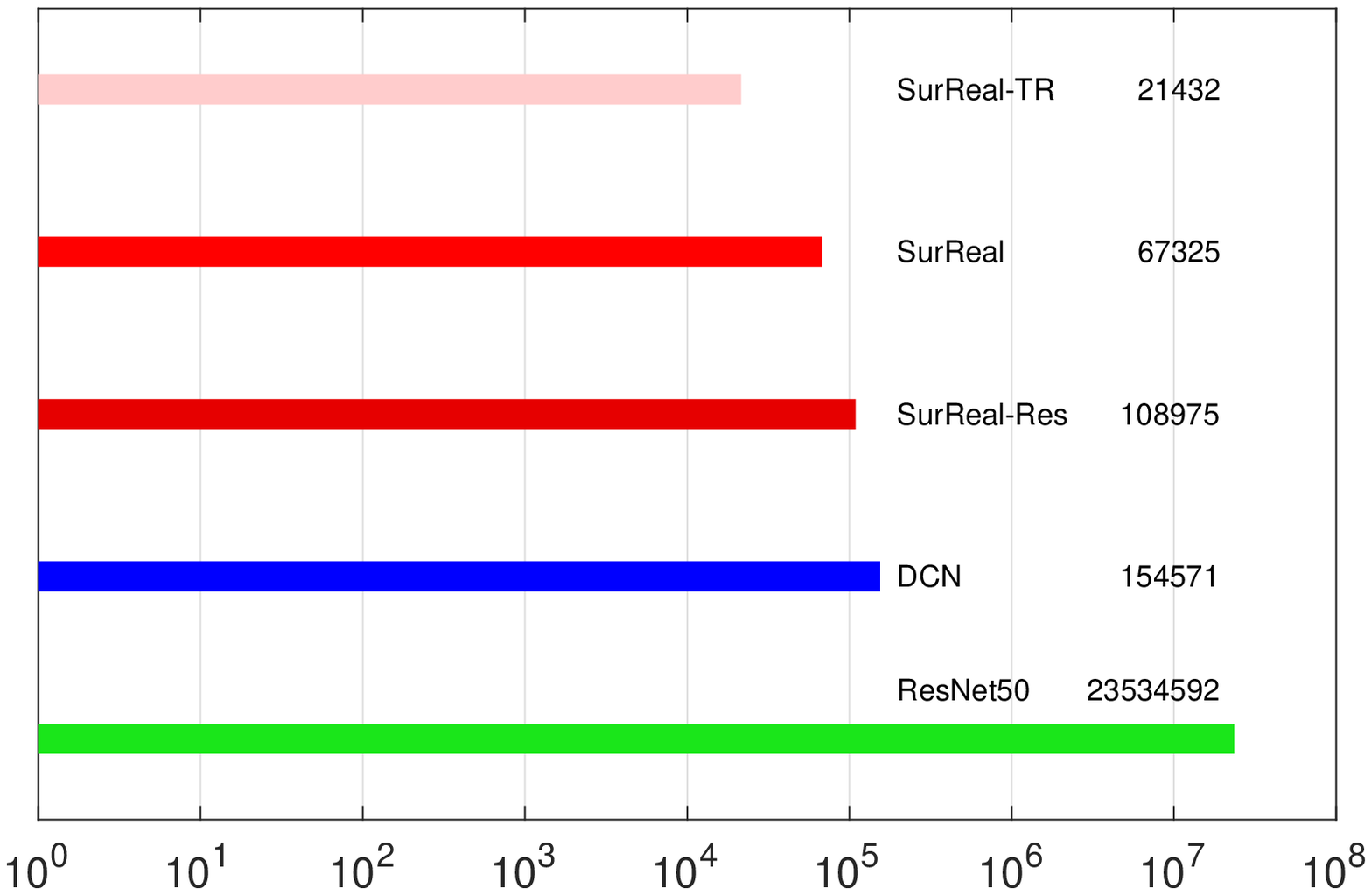}
\caption{
Our SurReal models for MSTAR are much leaner than baselines.
Each model size is plotted as a horizontal bar on a log scale,  labeled with the model name and the number of parameters on the right.
ResNet50 is the real-valued baseline and the largest  with 23.5M parameters.  DCN is the complex-valued baseline,  with 154K parameters at 0.7\% of ResNet50.  Our SurReal residual network has 109K parameters at 0.5\% (71)\% of ResNet50 (DCN).   Our basic SurReal CNN has 67K parameters at 0.3\% (44\%) of ResNet50 (DCN).  With the tensor ring  implementation for convolutions, our SurReal CNN could be further reduced to 21K parameters at 0.1\% (14\%) of ResNet50 (DCN).  }
\label{fig:MstarModelSizes}
\end{figure}
}
    
\def\figMstarAccuA#1{
\begin{figure}[#1] \centering
\includegraphics[width=0.48\textwidth]{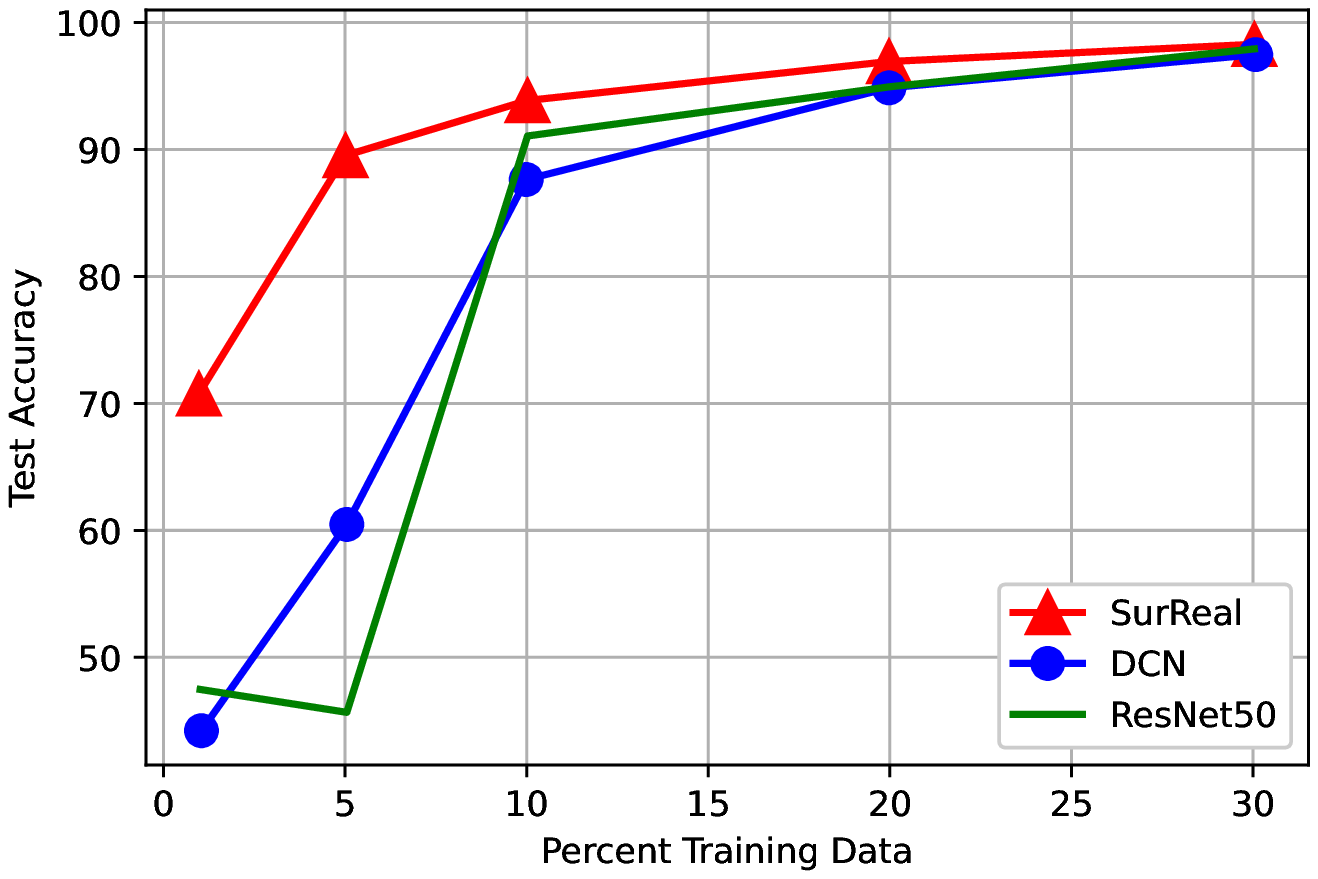}\\
\caption{Our SurReal CNN significantly outperforms real-valued and complex-valued baselines on MSTAR target recognition, when less data is used for training.  We split the data into  training and test sets at varying proportions and plot the test accuracy.  When there is enough training data, e.g., at 30\% training and 70\% testing, all three models perform similarly with  98\% accuracy.  When there is less training data, e.g., at 5\% training and 95\% testing, SurReal is much better with a test accuracy of 90\% over DCN's 60\% and ResNet50's  45\%, demonstrating the effectiveness of our complex-valued model.
}
\label{fig:MstarAccuA}
\end{figure}
}

\def\tabSurRealNet#1{
\begin{table}[#1]\centering
\caption{SurReal CNN in Detailed Layer Specification}
\setlength{\tabcolsep}{7pt}
\begin{tabular}{@{}lrccr@{}}
 \toprule
 {\bf Layer Type} & {\bf Input Shape} & {\bf Kernel} & {\bf Stride} & {\bf Output Shape}\\
  \midrule
\bf\color{Red} wFM CONV&\([2,1,100,100]\)	& $5\times 5$	&2   &\([2,20,48,48]\)\\
\midrule 
\bf\color{RoyalBlue} $G$-transport&\([2,20,48,48]\)   & -         & -       &\([2,20,48,48]\)\\
\midrule 
\bf\color{Red} wFM CONV &\([2,20,48,48]\)   & $5\times 5$   &2   &\([2,20,22,22]\)\\
\midrule 
\bf\color{RoyalBlue} $G$-transport&\([2,20,22,22]\)   & -         & -       &\([2,20,22,22]\)\\
\midrule 
\bf\color{Red} DIST &\([2,20,22,22]\) & - & -&\([20,22,22]\)\\
\midrule 
\bf\color{Cyan}  CONV &\([20,22,22]\)     &$5\times 5$     &1   &\([30,18,18]\)\\
\midrule 
\bf\color{gray} BN$+$ReLU&\([30,18,18]\)     & -         & -       &\([30,18,18]\)\\
\midrule 
\bf\color{Orange} MaxPool          &\([30,18,18]\)     &$2\times 2$     &2   &\([30,9,9]\)  \\
\midrule 
\bf\color{Cyan} CONV            &\([30,18,18]\)     &$5\times 5$     &3   &\([30,2,2]\)  \\
\midrule 
\bf\color{gray} BN$+$ReLU&\([30,2,2]\)       & -         & -       &\([30,2,2]\)  \\
\midrule 
\bf\color{Cyan} CONV             &\([30,2,2]\)       &$2\times 2$     &1   &\([30,1,1]\)  \\
\midrule 
\bf\color{gray} BN$+$ReLU&\([30,1,1]\)       & -         & -       &\([30,1,1]\)  \\
\midrule 
\bf\color{Cyan} FC           &\([30]\)           & -         & -       &\([50]\)      \\
\midrule 
\bf\color{Cyan} FC            &\([50]\)           & -         & -       &\([11]\)      \\
  \bottomrule
 \end{tabular}\\[5pt]
\begin{tabular}{@{}p{0.48\textwidth}@{}}
The network has two layers of complex-valued convolution (wFM)
and nonlinear activation ($G$-transport), distance
transform (DIST), and then a real-valued CNN classifier composed with standard convolution (CONV),
batch normalization (BN) and ReLU, max pooling (MaxPool), and fully
connected (FC) layers.  An input or output shape of 4 dimensions indicates 
complex-valued data,  with each complex number represented by two values: magnitude and phase.
For example, $[2,1,100,100]$ means a complex-valued 1-channel
$100\times 100$ image.   A shape of 3 dimensions indicates 
real-valued data.  In our network, DIST is the depth layer that
turns a complex-valued representation into a real-valued one; it separates the
complex-valued and real-valued representations in the SurReal CNN.
\end{tabular}
\label{tab:SurRealNet} 
\end{table}
}

\def\tabSurRealResNet#1{
\begin{table}[#1]\centering
\caption{SurReal Residual CNN in Detailed Layer Specification}
\setlength{\tabcolsep}{1pt}
\begin{tabular}{@{}lrcccr@{}}
 \toprule
 {\bf Layer Type} & {\bf Input Shape 1} &  {\bf Input Shape 2} & {\bf Kernel} & {\bf Stride} & {\bf Output Shape}\\
 \midrule
\bf\color{Red} wFM CONV&\([2,1,100,100]\) & -& $5 \times 5$	   &2   &\([2,20,48,48]\)\\
\midrule 
\bf\color{RoyalBlue}  $G$-transport&\([2,20,48,48]\) &  -           & -     &-   &\([2,20,48,48]\)\\
\midrule 
\bf\color{Red}  wFM CONV&\([2,20,48,48]\)   &-  & $5 \times 5$    &2   &\([2,20,22,22]\)\\
\midrule
\bf\color{Green} wFM Res   &\([2,20,48,48]\)   & \([2,20,22,22]\)& $5 \times 5$    &2   &\([2,20,22,22]\)\\
\midrule 
\bf\color{RoyalBlue} $G$-transport&\([2,20,22,22]\)          & -            & -         & -       &\([2,20,22,22]\)\\
\midrule 
\bf\color{Red} DIST &\([2,20,22,22]\)       & -            & -         & -       &\([20,22,22]\)\\
\midrule 
\bf\color{Cyan}  CONV             &\([20,22,22]\)     & -            &$5 \times 5$     &1   &\([30,18,18]\)\\
\midrule 
\bf\color{gray}  BN$+$ReLU&\([30,18,18]\)     & -            & -         & -       &\([30,18,18]\)\\
\midrule 
\bf\color{Cyan}  CONV Stack             &\([30,18,18]\)     & -            &$1,3,1$     &1   &\([30,18,18]\)\\
\midrule 
\bf\color{Green}  CONV Res&\([30,18,18]\)     &\([30,18,18]\)& -         & -       &\([40,18,18]\) \\
\midrule 
\bf\color{Orange} Maxpool          &\([40,18,18]\)     & -            &$2 \times 2$     &2   &\([40,9,9]\)  \\
\midrule 
\bf\color{Cyan}  CONV   &\([40,18,18]\)     & -            &$5 \times 5$     &3   &\([50,2,2]\)  \\
\midrule 
\bf\color{gray}   BN$+$ReLU&\([50,2,2]\)       & -            & -         & -       &\([50,2,2]\)  \\
\midrule 
\bf\color{Cyan}  CONV Stack &\([50,2,2]\)     & -            &$1,3,1$     &1   &\([50,2,2]\)\\
\midrule 
\bf\color{Green}  CONV Res&\([50,2,2]\)     &\([50,2,2]\)& -         & -       &\([60,2,2]\) \\
\midrule 
\bf\color{Cyan}  CONV             &\([60,2,2]\)       & -            &$2 \times 2$     &1   &\([70,1,1]\)  \\
\midrule 
\bf\color{gray}  BN$+$ReLU&\([70,1,1]\)       & -            & -         & -       &\([70,1,1]\)  \\
\midrule 
\bf\color{Cyan}   FC          &\([70]\)           & -            & -         & -       &\([30]\)      \\
\midrule 
\bf\color{Cyan}   FC           &\([30]\)           & -            & -         & -       &\([11]\)      \\
  \bottomrule
\end{tabular}
\\[5pt]
\begin{tabular}{@{}p{0.48\textwidth}@{}}
Our SurReal residual CNN utilizes both complex-valued residual blocks (Fig. \ref{fig:ResCNN}) as well as real-valued residual blocks.   
In CONV Stack, we stack three convolutions with $1\times 1$, $3\times 3$, $1\times 1$ kernels respectively. 
We zero pad the inputs on the $3\times 3$ convolution to preserve the spatial dimensions. 
\end{tabular}
 \label{tab:SurRealResNet} 
\end{table}
}

\def\figMstarSets#1{
\begin{figure}[#1]
    \begin{center}
    \setlength{\tabcolsep}{1pt}
    \begin{tabular}{cc}
    \includegraphics[trim=10 50 40 50, width=0.24\textwidth]{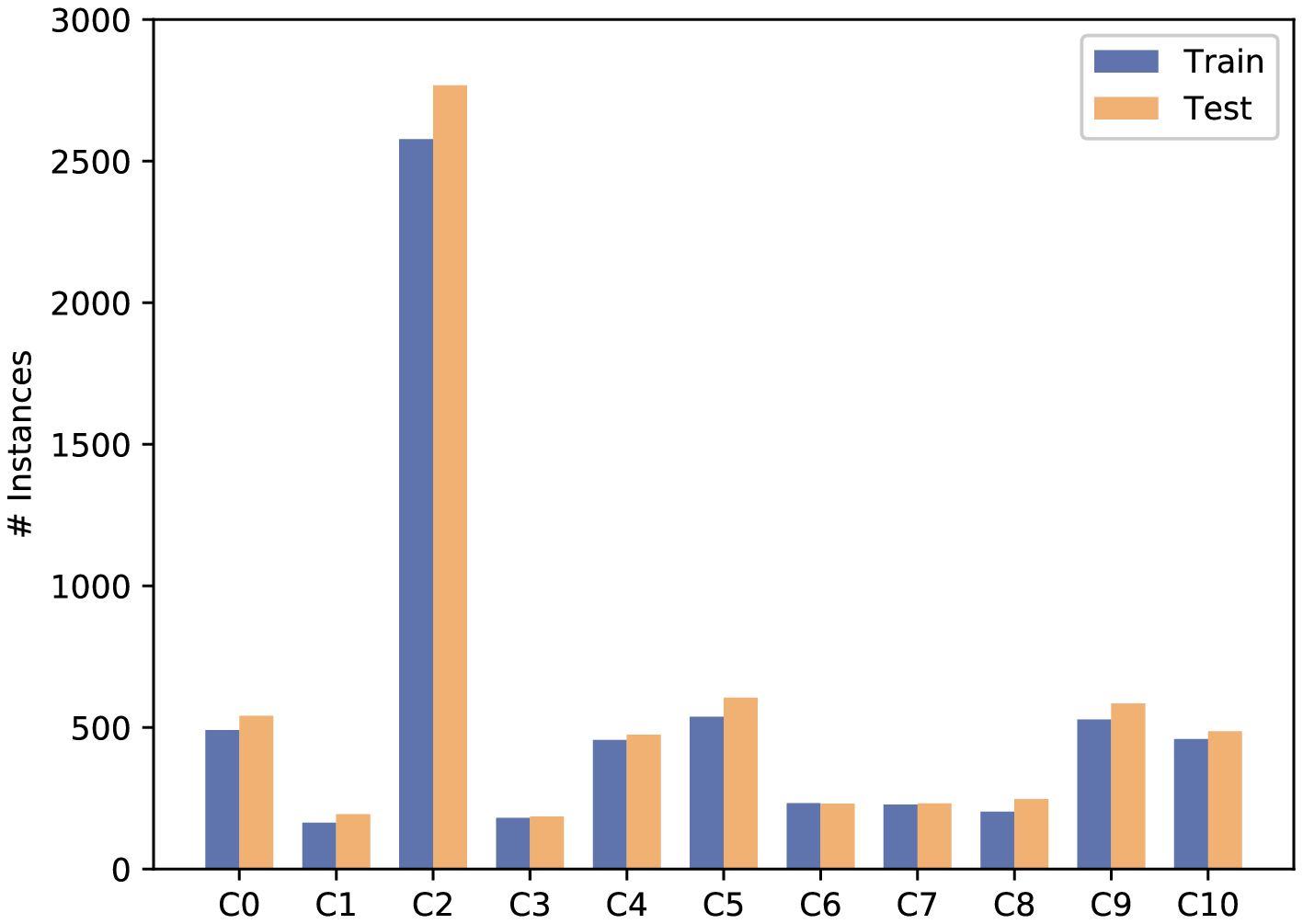} &
    \includegraphics[trim=10 50 40 50, width=0.24\textwidth]{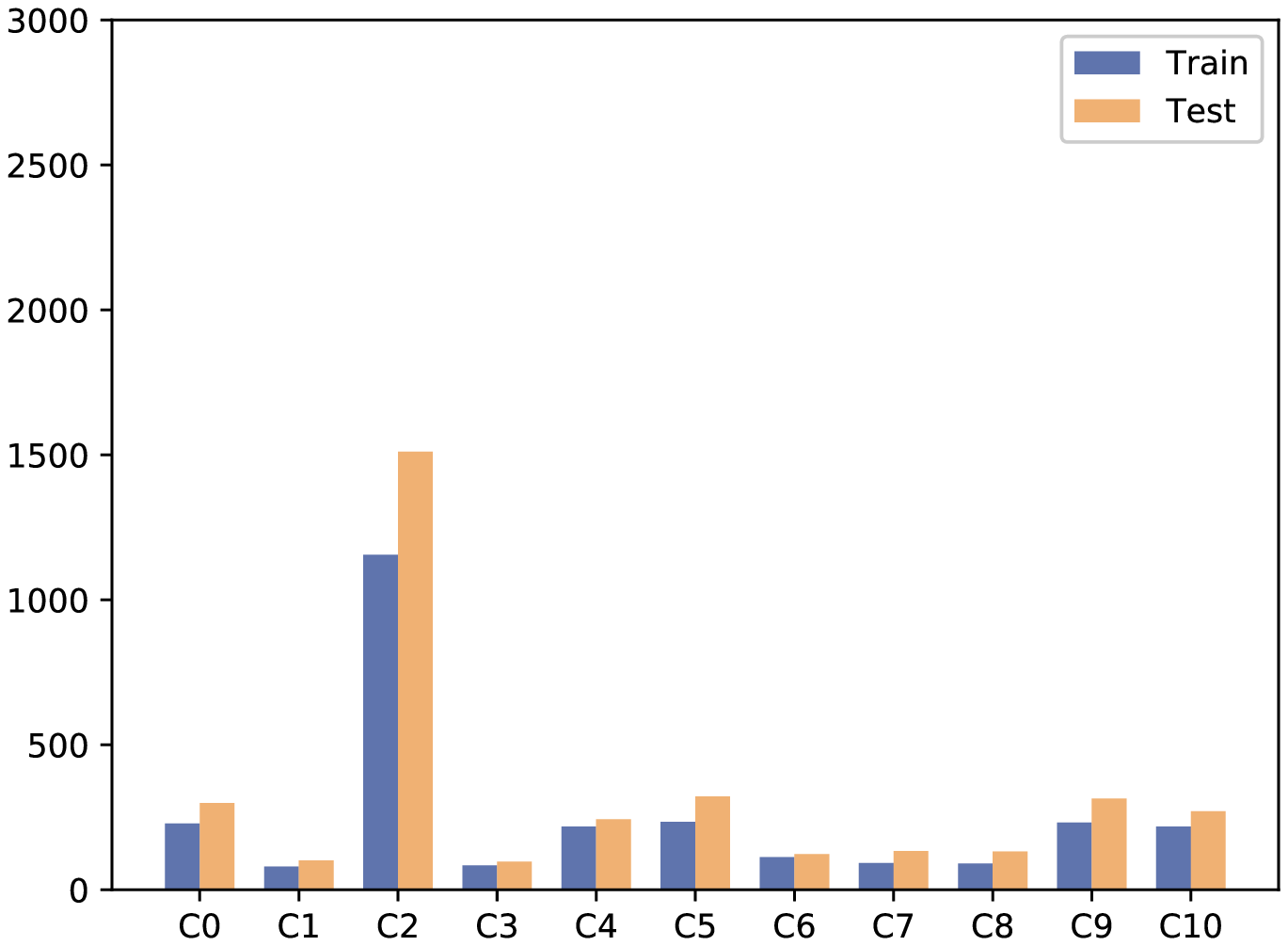} \\
    \end{tabular}
    \end{center}
    \caption{MSTAR large (L) and small (S) subsets have highly imbalanced classes.  C2 is the largest class,  C0, C4, C5, C9, C10 come next at about 20\% of the size of C2, and C1, C3, C6, C7, C8 are at about 8\% of the size of C2.}
    \label{fig:mstarSets}
\end{figure}
}

\def\figRadioML#1{
\begin{figure}[!ht]
    \centering
    \includegraphics[width=0.48\textwidth]{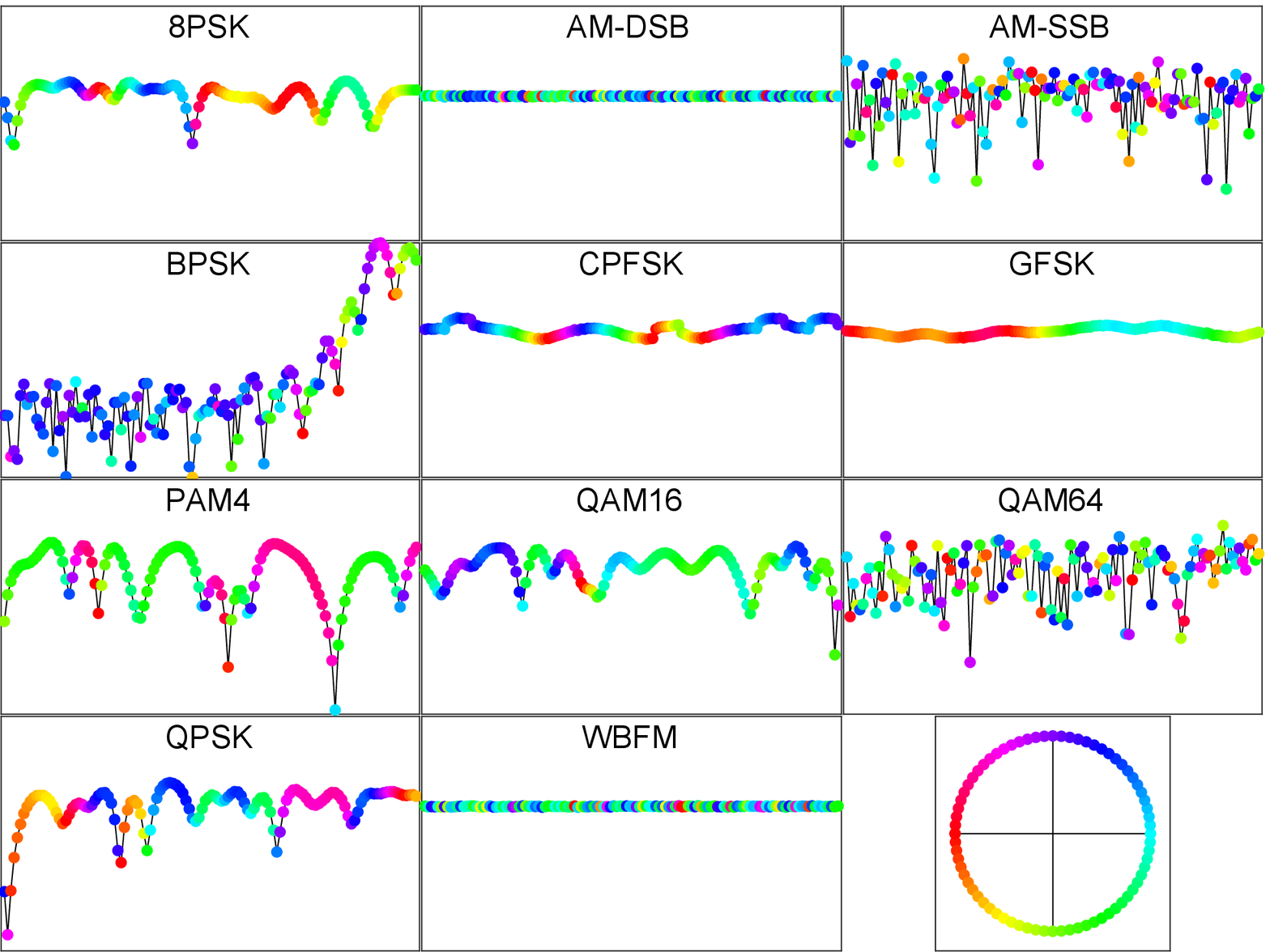}
    \caption{RadioML sample instances.  Each instance is a 128 time-step complex-valued signal at SNR 10.  We plot each signal as a line of 128 connected colored dots over time ($x$-axis): the height of the dot indicates the magnitude  on a log scale ($y$-axis), and the color of the dot indicates the phase with the HSV colormap.  The first 11 plots show one instance per class with the class name labeled on the top.  The last  plot visualizes the complex number $\{e^{i\,\theta},\theta\in[-\pi,\pi)\}$ as our colored dots, illustrating how the color varies with the phase $\theta$.    
    The shape of RF signal reflects both the message it is carrying and the modulation mode it is subject to.  The classifier must  ignore the distinction in the message but focus on the distinction in the modulation.
    }
\label{fig:RadioML}
\end{figure}
}

\def\figRFSNR#1{
\begin{figure}[#1]
    \includegraphics[width=0.48\textwidth]{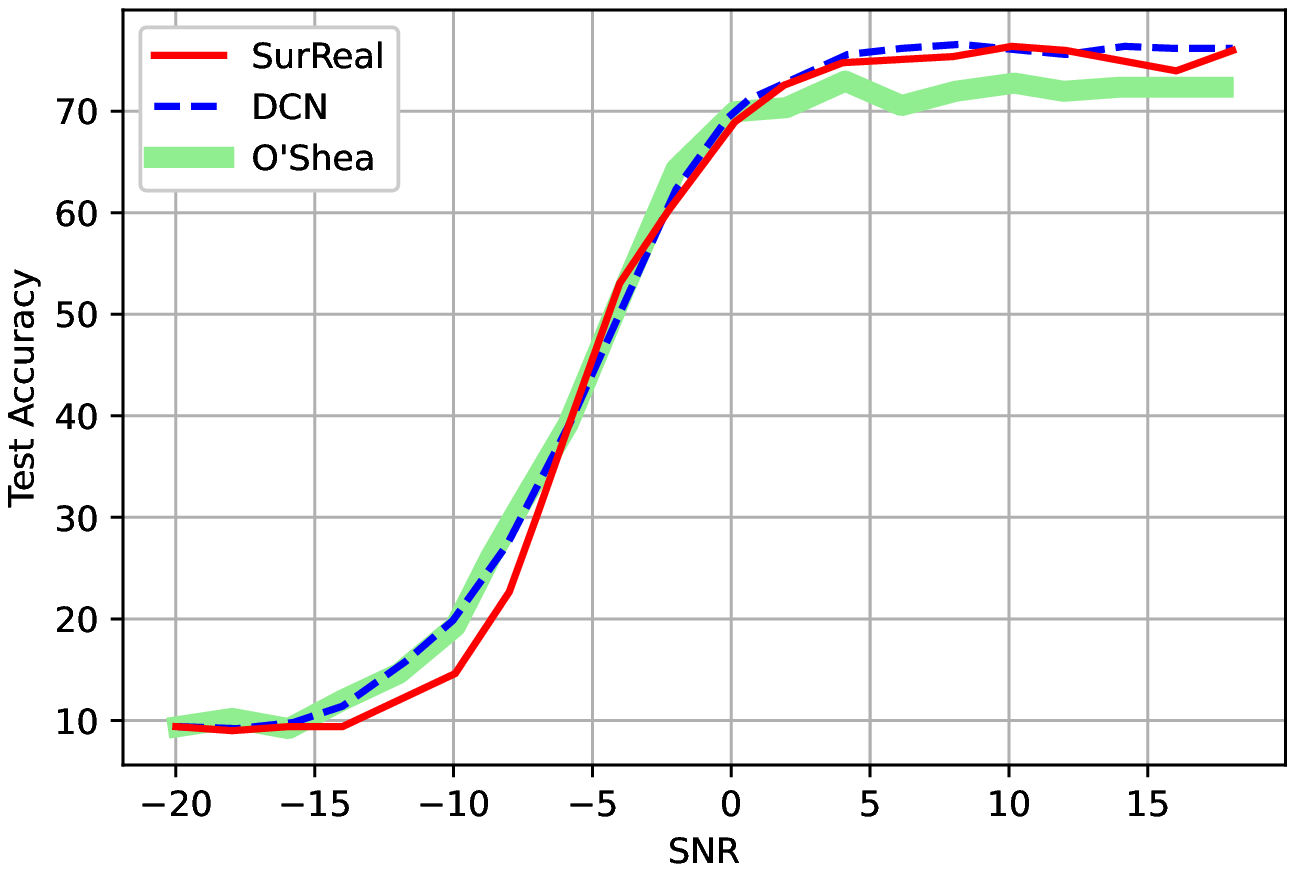}
    \caption{Test accuracy comparison on RadioML with a varying SNR.  All the models perform poorly at a very low SNR.  Our SurReal underperforms the two baselines for SNR in $[-15,-5]$.  The complex-valued baseline DCN is on par with the real-valued baseline O'Shea at a negative SNR, but better at a positive SNR.  Likewise, our SurReal outperforms the real-valued baseline at a positive SNR, at only 0.7\% (58\%) size of O'Shea (DCN).  }
    \label{fig:RFSNR}
\end{figure}
}

\def\figRFModelSize#1{
\begin{figure}[#1]\centering
\includegraphics[width=0.48\textwidth]{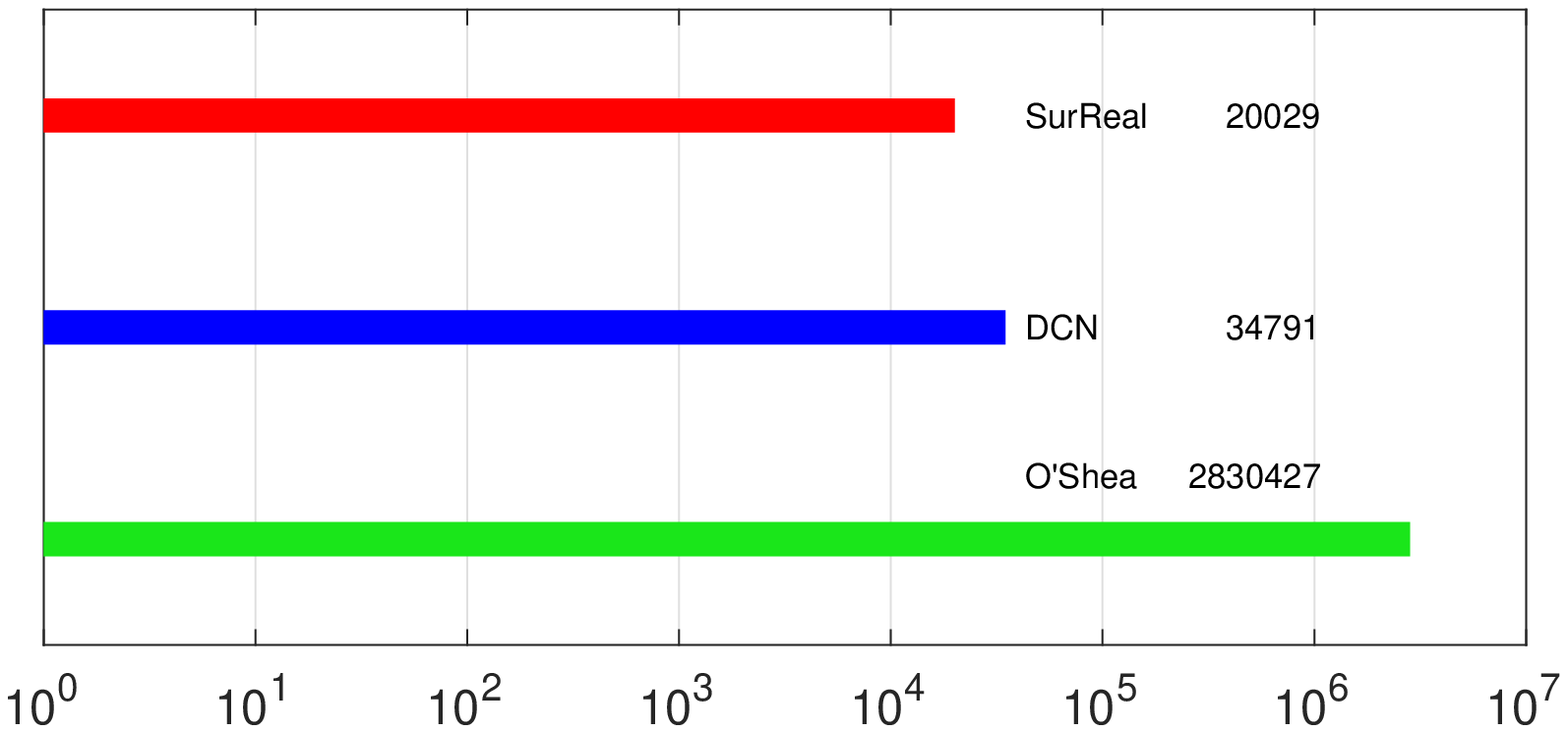}
\caption{
Our SurReal model for RadioML is much leaner than baselines.
Each model size is plotted as a horizontal bar on a log scale,  labeled with the model name and the number of parameters on the right.
O'Shea's model is the real-valued baseline and the largest  with 2.8M parameters.  DCN is the complex-valued baseline,  with 35K parameters at 1\% of O'Shea.  Our SurReal CNN has 20K parameters at 0.7\% (58\%) of O'Shea (DCN).   
}
    \label{fig:RFModelSize}
\end{figure}
}

\tabSurRealNet{t}
\tabSurRealResNet{t}

\section{Experiments}

We compare our SurReal complex-valued classifier
against two baselines.  {\bf 1)} The first baseline is a real-valued CNN classifier such as ResNet50 which ignores the geometry of complex numbers and treats each complex value as two independent real numbers. 
{\bf 2)} The second baseline is the deep complex networks (DCN) which extends real-valued CNN layer functions to the complex domain by the form of the functions such as complex-valued convolution, batch-normalization, nonlinear activation, and weight initialization \citep{trabelsi2017deep}.
The specific DCN models used in \citep{trabelsi2017deep} are very large, on the order of one million parameters.  To help directly compare complex-valued layer functions, we adopt the same SurReal Residual architecture for the DCN baseline, but replacing all the convolution layers (complex-valued wFM and real-valued convolution) and nonlinear activation functions with DCN's proposed counterparts. 

We experiment on two complex-valued datasets: SAR image dataset MSTAR \citep{keydel1996mstar} and 
synthetic RF signal dataset RadioML  \citep{o2016convolutional}.  
All the models are trained on a GeForce RTX $2080$ GPU for $120$ epochs, using Adam optimizer and cross-entropy loss. 


\tabMstarSize{tp}
\figMstarModelSizes{tp}

\subsection{MSTAR Classification}

\noindent
{\bf MSTAR dataset.}  
There are a total of $15,716$ complex-valued X-band SAR images, distributed unevenly over 11 classes:  The first 10 classes contain different target vehicles and the last 1 class contains background clutter.  See Table \ref{tab:MstarSize} for the total number of images per class.
We take the $100\times 100$ center crop of each image and convert the complex value of each pixel into the polar form.

\noindent
{\bf Real-valued CNN baseline.}  
ResNet50 \citep{he2016deep} is widely successful on real-valued image classification and it is also used as a baseline in \citep{shao2018lightweight}.

\noindent
{\bf Two SurReal CNN architectures.}  
Table \ref{tab:SurRealNet} lists detailed layer specification of our basic model.  Table \ref{tab:SurRealResNet} adds residual connections.  Both models have two complex-valued convolutions (with nonlinear activation), one distance transform layer, and two real-valued convolutions (with batch normalization and ReLU), max pooling, and two fully connected layers.
While we have listed $G$-transport in the two tables, we have also tried tReLU as the complex-valued nonlinear activation function; their difference is insignificant in initial experiments on the MSTAR dataset, and we focus on G-transport for its simplicity.

\noindent
{\bf Model size comparison.}  
While ResNet50 and DCN have $23$ million  and 155K parameters respectively, our SurReal CNN has 67K parameters and SurReal residual CNN has 109K parameters.  We can further reduce the parameter count by implementing convolutions with tensor ring decomposition \citep{oseledets2011tensor,zhao2016tensor}.
Fig. \ref{fig:MstarModelSizes} plots these model sizes on a log scale.  The saving is substantial: our SurReal CNN  is less than $0.1\%$ of the real-valued baseline and $44\%$ of the complex-valued baseline.

\figMstarAccuA{tp}

\noindent
{\bf Task 1: 10-class target recognition.}  
For the 10 target classes, we split all the data in 5 varying proportions of 1\%, 5\%, 10\%, 20\%, 30\% for training and the rest for testing.   Fig. \ref{fig:MstarAccuA} shows that our SurReal significantly outperforms DCN and ResNet50, especially when a small percentage of training data is used.  At 5\% training and 95\% testing, the accuracy is 90\% for SurReal, 60\% for DCN, and 45\% for ResNet50.

\figMstarSets{tp}

\noindent
{\bf Task 2: 11-class classification.}
We also include the remaining clutter class which contains miscellaneous background images.  We create two random subsets, large (L) and small (S), and the small set of $6,295$ images are contained  entirely in the large set of $12,610$ images.
Fig. \ref{fig:mstarSets} shows the number of instances across 11 classes and in training/testing splits.

\tabMstarAccuB{tp}
\figMstarConf{tp}
\figMstarConvResponse{tp}

Table \ref{tab:MstarAccuB} shows that all the models perform at a high accuracy of 99\% for the large dataset.  The performance drops as the overall data size is reduced by half in the small dataset, but the drop is the least at $-0.8\%$ for SurReal-Res, followed by $-1.4\%$ for SurReal, $-1.7\%$ for ResNet50, and the most at $-6.6\%$ for DCN.  Our results against the two baselines suggest that it is both the residual connections and more importantly how we handle complex-valued data that delivers more generalizing performance from smaller training data.

Fig. \ref{fig:MstarConvResponse} shows sample channel responses from 
our SurReal CNN on MSTAR-S images.  With two complex-valued wFM convolutions, followed by distance transform and real-valued convolution, the representation for each input SAR image quickly becomes more distinctive across classes, facilitating accurate discrimination.

Fig. \ref{fig:MstarConf} shows the confusion matrix between classes on MSTAR-S.  In general, the more training instances in the class, the least confusion with other classes at the test time.  However, despite the significant class imbalance, the performance gap is small between the minority and majority classes.  Residual connections help clear up more confusion.

\subsection{RadioML Classification}

\figRadioML{tp}
\figRFModelSize{tp}
\figRFSNR{tp}

\noindent
{\bf RadioML Dataset.}
They are synthetically generated radio signals with modulation operating over both voice and text data.  Noise is added further for channel effects. 
Each signal has 128 time samples and is tagged with a signal-to-noise ratio (SNR), in the range of $[-20,18]$ with an increment step of $2$. 
There are $11$ modulation modes, of which BPSK, QPSK, 8PSK, 16QAM, 64QAM, BFSK, CPFSK, PAM4 are digital modulations, and WB-FM, AM-SSB, and AM-DSB are analog modulations.  
There are $20,000$ instances per modulation.   
See sample instances in Fig. \ref{fig:RadioML}.
The data is split $50/50$ between training and testing.

\noindent
{\bf  Real-valued CNN Baseline.}
We use O'Shea's model \citep{o2016convolutional} and feed the 1-channel complex-valued RF signal as a (real,imaginary) two-channel signal.

\noindent
{\bf Model size comparison.}  
We follow the architecture of DCN and SurReal for images and adapt the spatial dimensions to fit the $1\times 128$ RF signals.  Fig. \ref{fig:RFModelSize}
plots these model sizes on a log scale. 
Our SurReal CNN is $0.7\%$ of the real-valued baseline and $58\%$ of the complex-valued baseline.

\noindent
{\bf Accuracy over SNR.}
Fig. \ref{fig:RFSNR} compares the test accuracy at various SNR levels.  Our SurReal underperforms the baselines at lower SNRs and outperforms the real-valued O'Shea baseline at higher SNRs.  For example at SNR 10, it achieves $76.1\%$, surpassing O'Shea's $72.7\%$ with $0.7\%$ of its model size and on par with DCN's $76.3\%$ at $58\%$ of its size.

\section{Summary and Conclusions}
Deep learning is widely adopted in machine learning and computer vision.  Most existing deep learning techniques are developed for data in a vector space.  However, practical data often have correlations between channels and are better modeled as points a manifold.

While the Nash embedding theorem \citep{boothby1986introduction} assures us that it is always feasible to embed the data on a manifold into a higher dimensional vector space,
it would also result in an increase in the model complexity and training time.  Recent geometric deep learning approaches develop tools for spaces with certain geometry such as graphs and surfaces.  

We deal with deep learning on complex-valued data, and we approach it from a geometric perspective.  The common approach is to represent complex-valued data as two-channel real-valued data and then all the real-valued deep learning tools can be used.  However, this representation ignores the underlying geometry that defines the complex-valued data:  Complex-valued data containing the same information could be subject to arbitrary complex-valued scaling.

We propose to model the space of complex numbers as a product manifold of non-zero scaling and planar rotations.  Arbitrary complex-valued scaling naturally becomes a group of transitive actions on this manifold.  We can subsequently define convolution on the manifold that is {\it equivariant} to this action group, and define distance transform that is {\it invariant} to this action group.  The manifold perspective also allows us to define new nonlinear activation functions such as tangent ReLU and $G$-transport, as well as residual connections on the manifold-valued data.  

A complex-valued CNN classifier composed of such layer functions has built-in invariance to complex-valued scaling, so that the model only needs to focus on the discrimination between classes.  Our experimental results validate our principled approach, and we dub our model {\it SurReal} based on its high performance achieved at a super-lean model size compared with real-valued or complex-valued baselines.

\ifCLASSOPTIONcaptionsoff
  \newpage
\fi

\section*{Acknowledgements}
This research was supported, in part, by Berkeley Deep Drive and DARPA.
An earlier version of a few diagrams was constructed with help from Liu Yang.
We thank Utkarsh Singhal for thoughtful proofreading and anonymous reviewers for critical and constructive comments.
{
\small
\bibliographystyle{plainnat}
\bibliography{references.bib}
}

\end{document}